%% file: arxiv_revised.tex
\documentclass{article}

\usepackage{arxiv}
\usepackage[utf8]{inputenc} % allow utf-8 input
\usepackage[T1]{fontenc}    % use 8-bit T1 fonts
\usepackage{hyperref}       % hyperlinks
\usepackage{url}            % simple URL typesetting
\usepackage{booktabs}       % professional-quality tables
\usepackage{amsfonts}       % blackboard math symbols
\usepackage{nicefrac}       % compact symbols for 1/2, etc.
\usepackage{microtype}      % microtypography
\usepackage{xcolor}         % colors
\usepackage{wrapfig}
\usepackage{mathtools}
\usepackage{comment}

\input{math_commands.tex}

\usepackage{url}
\usepackage{graphicx}
\usepackage{multirow}
\usepackage{adjustbox}

\usepackage{amsfonts}
\usepackage{amsmath}
\usepackage{algorithm}
\usepackage{algpseudocode}

\usepackage{bbm}
\usepackage{microtype}
\usepackage{graphicx}
\usepackage{subfigure}
\usepackage{natbib}
\usepackage{booktabs} % for professional tables
\usepackage{amsfonts}       % blackboard math symbols
\usepackage{nicefrac}       % compact symbols for 1/2, etc.
\usepackage{subfiles}
\usepackage{xspace}
\usepackage{float}
\usepackage{euscript}
\usepackage{amsthm}
\usepackage{cleveref}
\usepackage{thmtools}
\usepackage{thm-restate}
\usepackage{soul}
\usepackage{tikz}

\newtheorem{theorem}{Theorem}
\newtheorem{proposition}{Proposition}

\newtheorem{definition}{Definition}

\newtheorem{lemma}[theorem]{Lemma}
\newcommand{\renyi}{R\'enyi\xspace}
\newcommand{\Renyi}{R\'enyi\xspace}

\definecolor{DarkGreen}{rgb}{0.1,0.5,0.1}

\newtoggle{showcomments}
\toggletrue{showcomments} % Set to \togglefalse{showcomments} to hide comments
\newtoggle{showupdates}
\toggletrue{showupdates} % Set to \togglefalse{showupdates} to show updates in black text
\graphicspath{ {./figs/} }

\newcommand{\DPSGD}{DP-SGD}

\newcommand{\Laplace}{Laplace}
\newcommand{\Gaussian}{Gaussian}

\newcommand{\Argmax}{Argmax}

\definecolor{DarkGreen}{rgb}{0.1,0.5,0.1}

\title{Beyond Laplace and Gaussian: Exploring the Generalized Gaussian Mechanism for Private Machine Learning}

\author{%
  Roy Rinberg\thanks{Funded in part by NSF grant BCS-2218803 (HNDS-I), the Digital Data Design Institute at Harvard, and Coefficient Giving. Part of this work completed while R.R. was a student at Columbia University and visiting University of Toronto.  }\\
  Harvard University\\
  \texttt{royrinberg@g.harvard.edu} \\
  \And%
  Ilia Shumailov\\
  University of Oxford \\ \texttt{ilia.shumailov@chch.ox.ac.uk}
  \And
  Vikrant Singhal\thanks{VS is supported in part by NSF grant BCS-2218803, the Digital Data Design Trustworthy AI Lab at Harvard, and a grant from the Sloan Foundation.}\\
  Harvard University\\
  \texttt{vikrant@seas.harvard.edu}
  \And
  Rachel Cummings\thanks{R.C. supported in part by NSF grant CNS-2138834 (CAREER)}\\
  Columbia University\\
  \texttt{rac2239@columbia.edu} \\
  \And%
  Nicolas Papernot\thanks{N.P. supported by a Canada CIFAR AI Chair at the Vector Institute.} \\
  University of Toronto \& Vector Institute\\
\texttt{nicolas.papernot@utoronto.ca}
}

\begin{document}

\maketitle

\begin{abstract}
Differential privacy (DP) is obtained by randomizing a data analysis algorithm, which necessarily introduces a tradeoff between its utility and privacy. Many DP mechanisms are built upon one of two underlying tools: Laplace and Gaussian additive noise mechanisms. We expand the search space of algorithms by investigating the Generalized Gaussian (GG) mechanism, which samples the additive noise term $x$ with probability proportional to $e^{-\frac{| x |}{\sigma}^{\beta} }$ for some $\beta \geq 1$ (denoted $GG_{\beta, \sigma}(f,D)$). The Laplace and Gaussian mechanisms are special cases of GG for $\beta=1$ and $\beta=2$, respectively. 

We prove that the full GG family satisfies differential privacy and extend the PRV accountant to support privacy loss computation for these mechanisms. We then instantiate the GG mechanism in two canonical private learning pipelines, PATE and DP-SGD. Empirically, we explore PATE and DP-SGD with the GG mechanism  across the computationally feasible values of $\beta$: $\beta \in [1,2]$ for DP-SGD and $\beta \in [1,4]$ for PATE. For both mechanisms, we find that $\beta=2$ (Gaussian) performs as well as or better than other values in their computational tractable domains.This provides justification for the widespread adoption of the Gaussian mechanism in DP learning.

\end{abstract}

\section{Introduction} \label{sec:intro}
\input{sections_revised/01-intro}

\section{Preliminaries} \label{sec:background}
\input{sections_revised/02-background}

\section{Generalized Gaussian Mechanism and Privacy Guarantees} \label{sec:theory}
\input{sections_revised/03-theory}

\section{GG Mechanism for PATE and Private \Argmax} \label{sec:argmax}
\input{sections_revised/04-argmax}

\section{GG Mechanism for Differentially Private Stochastic Gradient Descent} \label{sec:dpsgd}

\input{sections_revised/05-dpsgd}

\section{Conclusion} \label{sec:discussion}
In this work, we have studied the Generalized Gaussian Mechanism and its privacy guarantees  -- as well as its variants SGG, GGNmax, and $\beta$-DP-SGD. We have demonstrated the application of these mechanisms to private ML, particularly for PATE (via private \Argmax) and \DPSGD. 

While one may hope that providing an additional hyperparameter $\beta$ to control the noise distribution would allow for improvements in the privacy-accuracy tradeoff of DP mechanisms, our empirical evaluations instead reveal that optimizing $\beta$ has a relatively modest impact on test accuracy. In particular, values close to $\beta =2$ (i.e., Gaussian noise) exhibit near-optimal performance, which provides insight into the popularity of \Gaussian~noise in \DPSGD, PATE, and other private ML applications. 
 
We also showed how the PRV accountant can be used for Generalized Gaussian Mechanism and its variants, and that privacy accounting is dimension-independent for the $\beta$-GG Mechanism, using $\ell_\beta$ sensitivity. This dimension-independence has the potential to dramatically improve the computational costs of privacy accounting for this family of mechanisms.

An interesting extension for future work is that the GG Mechanism samples noise independently across dimensions. For the Laplace Mechanism ($\beta=1$), it is known that sampling from a high-dimensional Laplace variant can improve performance in private ML settings such as private empirical risk minimization \citep{CMS11}. Multi-dimensional Gaussian distributions are the only spherically symmetric distribution where all the component random variables are independent \cite{gaussian_independence}, so such high-dimensional variants would not improve performance for $\beta=2$. This suggests that for $\beta \in [1,2)$, it may be possible to improve utility for the same privacy guarantee by sampling from a single high-dimensional distribution rather than sampling independently for each coordinate.

\bibliography{citations}

\bibliographystyle{plainnat}

\appendix

\newpage

\input{sections_revised/appendix}

\end{document}

%% file: math_commands.tex
\usepackage{amsmath,amsfonts,bm}

\def\1{\bm{1}}

\def\eps{{\epsilon}}

\DeclareMathAlphabet{\mathsfit}{\encodingdefault}{\sfdefault}{m}{sl}
\SetMathAlphabet{\mathsfit}{bold}{\encodingdefault}{\sfdefault}{bx}{n}

%% file: sections_revised/01-intro.tex
As applications of machine learning (ML) often involve sensitive information, there is an increasing need to provide  privacy protections for the individuals whose data are included in the training datasets. Privacy concerns have prompted the development of privacy-preserving ML techniques, which aim to prevent the leakage of private information analyzed during training. One of the primary frameworks for achieving this goal is differential privacy (DP), a rigorous mathematical framework that provides quantifiable privacy guarantees \citep{original_DP}.

Two popular techniques for implementing DP in ML are Differentially Private Stochastic Gradient Descent (\DPSGD) \citep{dpsgd} and Private Aggregation of Teacher Ensembles (PATE) \citep{pate_2017}. Traditionally, \DPSGD~entails Poisson sampling from the dataset, gradient clipping, and then the addition of \Gaussian~noise to the gradient. PATE, on the other hand, involves training an ensemble of teacher models on disjoint subsets of the data, then privately aggregating the votes of the teacher models on a public dataset, in order to train a student model on the privately labeled public dataset. PATE achieves private vote aggregation through a private variant of \Argmax, which is obtained by adding noise to the vote counts, and then finding the \Argmax~of the noisy histogram.

Both \DPSGD~and PATE achieve privacy protection through mechanisms that add noise drawn from specific probability distributions, namely \Laplace~or \Gaussian. The choice of the noise distribution plays a crucial role in determining the privacy-accuracy tradeoffs of an algorithm, and algorithm designers often make problem-dependent decisions in choosing between \Laplace~and \Gaussian~Mechanisms. However, many of the underlying tradeoffs between these two discrete choices remain unclear. In this work, we explore a continuum of private mechanisms that extends these two special cases of noise distributions. We investigate the Generalized Gaussian Mechanism (GG) \citep{GG_mechanism_1}, denoted $GG_{\beta, \sigma}(f,D)$, which adds noise to the true function value $f(D)$ sampled from the Generalized Gaussian distribution,\footnote{Sometimes referred to in the literature as the Exponential Power Distribution or Generalized Normal Distribution} denoted $\mathcal{N}_\beta (\mu, \sigma)$, with probability density function (PDF), 
\begin{equation}\label{eq.ggpdf}
    p(x| \mu, \sigma, \beta) \propto e^{- \frac{| x -\mu |^{\beta}}{\sigma}}.
\end{equation}

Figure \ref{fig:equivalent_privacy_mechanisms__PRV} illustrates these PDFs with $\mu=0$ for different $\beta$ and $\sigma$ values on both linear and log scales. Interestingly, all of the noise distributions illustrated in Figure \ref{fig:equivalent_privacy_mechanisms__PRV} correspond to equivalent DP guarantees when used in the GG Mechanism (see \Cref{sec:mechanisms_with_equal_privacy} for more details). We observe that a larger $\beta$ value corresponds to a PDF that is more concentrated around the mean. In order to satisfy the same $(\epsilon, \delta)$-DP, $\sigma$ must be simultaneously increased to compensate for the lighter tail. 

\begin{figure}[H]\label{fig:equivalent_privacy_mechanisms__PRV}
    \centering
    \subfigure{\includegraphics[width=0.8\textwidth]{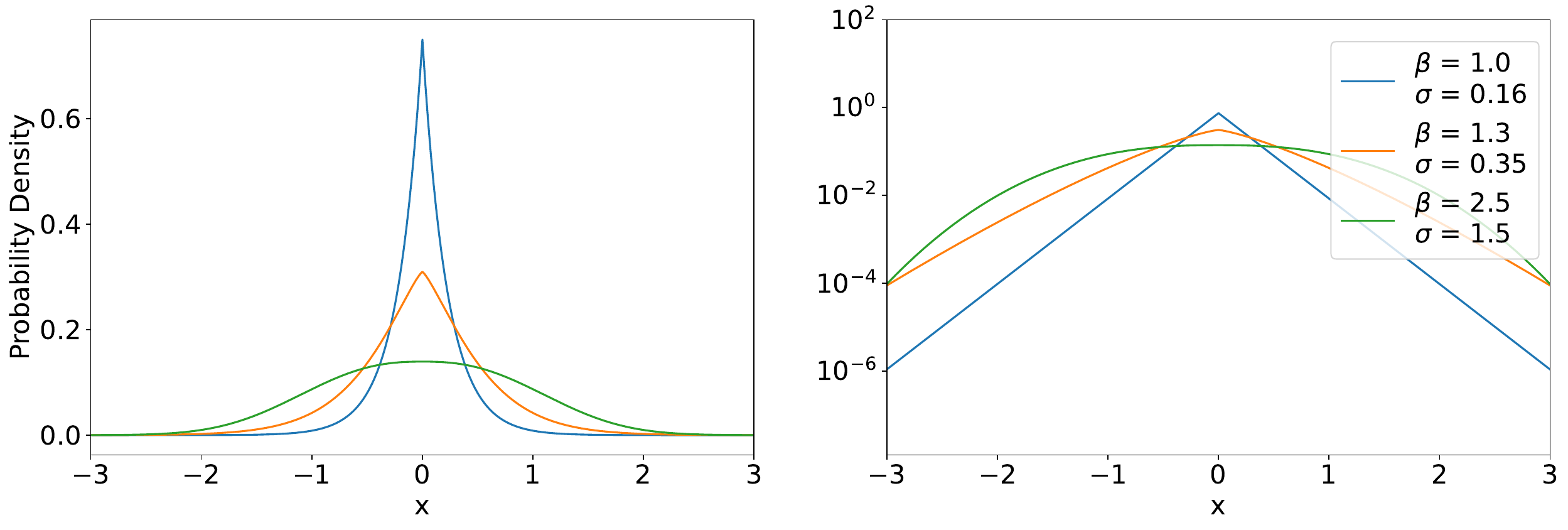}
    }
    \caption{\small{Linear (left) and log-scale (right) PDFs of the Generalized Gaussian distribution for GG mechanisms of varying $\beta$, which all satisfy the same $(\epsilon, \delta)$-DP privacy guarantee. }} 
\end{figure}

We focus on this family of mechanisms because it generalizes both the \Laplace~and \Gaussian~Mechanisms, which are special cases of the GG Mechanism for $\beta =1$ and $\beta = 2$, respectively. 

We explore two settings relevant to ML: PATE and \DPSGD. Our findings show that the choice of $\beta$ has a relatively small effect on test-accuracy. Within the computationally tractable range of $\beta$ for each mechanism, we find that $\beta = 2$ (Gaussian) performs as well as or better than other values. This helps explain why \Gaussian~noise may be so popular. Our analysis leads to the conclusion that for \(GG_{\beta, \sigma}(f,D)\) mechanisms, there is minimal or no improvement in model utility compared to the Gaussian mechanism.

\subsection{Related Work}

A number of prior works explored alternative DP mechanisms that extend the Laplace and Gaussian Mechanisms. The Staircase Mechanism~\citep{staircase_mechanism} was derived as an alternative to the \Laplace\ Mechanism by minimizing the variance of the noise distribution in order to improve utility guarantees. The Podium Mechanism~\citep{podium_mechanism} improved upon the Staircase Mechanism by changing the noise distribution to be defined over finite (truncated) support, instead of the infinite support used in other mechanisms like \Laplace. While both mechanisms have rigorous analysis of utility in the single-shot regime, they are neither studied nor optimized for high-composition regimes, and thus are not used in ML, which is the key use-case we investigate.

\citet{cactus_mechanism} developed the Cactus Mechanism to specifically address the high-composition regime by numerically computing a mechanism that minimizes the Kullback-Leibler divergence between the conditional output distributions of a mechanism given two different inputs, under a high number of compositions. \citet{cactus_mechanism} derived a near-optimal divergence loss in the high-composition regime; however, they only considered privacy for 1-dimensional outputs, making the result not directly applicable to ML models. Further, they focused on the optimal DP mechanism in the limit of a large number of compositions, explicitly ignoring the low-composition regime.

\citet{awan_CND_2022} address the multivariate setting by deriving a family of multivariate log-concave canonical noise distributions~\citep{awan2023canonical}, allowing addition of minimal noise for a particular f-DP guarantee. However, their work is constrained to $\ell_1$ and $\ell_{\infty}$ sensitivity settings and does not analyze the privacy-utility tradeoff of these mechanisms for private ML. Unfortunately, the proposed mechanisms are either only for Gaussian DP (known to underestimate privacy \citep{microsoft_PRV_2}), or are not easily integrated with existing privacy accountants in the subsampled regime.

\citet{GG_mechanism_1} introduced and partially analyzed the Generalized Gaussian Mechanism, which adds noise from the Generalized Gaussian distribution (Definition \ref{def.gengauss}). They provided probabilistic-DP (a variation of approximate DP) guarantees for $\beta = 1$ and $\beta \geq 2$, and gave empirical results for integer $\beta$-values in the GG mechanism applied to training Support Vector Machines on sanitized data from several tabular datasets. In a separate vein of research, \citet{certifiably_robust_RDP}, provided a relatively weak bound for the \Renyi Differential Privacy (RDP, Definition~\ref{def:RDP-calculation}) when $\beta >2$, for the limited use case of $\alpha =1$, which restricts much of the usefulness of the RDP formulation. To best of our knowledge, no previous work provides tight $(\epsilon, \delta)$-DP guarantees that are useful for the high composition regime. In our work, we consider $\beta \in [1,2]$ for DPSGD and $\beta \in [1,4]$ for PATE, and explore the privacy-accuracy tradeoff of the GG mechanism and its applications to PATE and~\DPSGD, and empirically provide tighter privacy guarantees than those provided by prior works, for any number of compositions.

\subsection{Our Contributions}

In this work, we investigate the Generalized Gaussian Mechanism, a little-explored family of DP mechanisms. Our main contributions are the following:
\begin{enumerate}
\item \textbf{GG-Mechanism and privacy accounting:}
We first show in Section \ref{sec:theory} that the GG Mechanism \(GG_{\beta, \sigma}(f,D)\) is DP for any $\beta$, and provide a PRV-based privacy accounting method for this mechanism. We present the Generalized Gaussian mechanism and its privacy guarantees in Section \ref{s.mechs}, and then in Section \ref{s.accounting}, we present a privacy accounting method which enables the GG-Mechanism to be used in high-composition ML applications. In Appendix \ref{appendix:subsampling}, we also introduce the Sampled Generalized Gaussian Mechanism (SGG), which is  a variant of the GG-Mechanism that first subsamples the database to make use of privacy amplification.

\item \textbf{GG-Mechanism for PATE and DP-SGD:}
In \Cref{sec:argmax}, we introduce the GGNMax algorithm for computing a private argmax for PATE, and in \Cref{sec:dpsgd}, we show how to use the Sampled GG Mechanism for \DPSGD\ in a new mechanism that we name $\beta$-Differentially Private Stochastic Gradient Descent ($\beta$-\DPSGD). We show that all four mechanisms satisfy DP (Theorems \ref{thm.ggprivacy}, \ref{thm.sggprivacy}, \ref{thm.ggargmaxpriv}, and \ref{thm:DPSGD_is_DP}), and we show how to extend PRV accountant to track privacy budget over many compositions of these mechanisms.

\item \textbf{Optimizing $\beta$ parameter for accuracy}:
   For guidance on choosing $\beta$, in \Cref{sec:argmax}, we empirically find that in privately computing an argmax in PATE, the choice of $\beta$ has a weak relationship with the accuracy. In \Cref{sec:dpsgd}, we empirically find that within the computationally tractable range of $\beta \in [1,2]$, the \Gaussian~mechanism ($\beta = 2$) performs as well as or better than other values when hyperparameters are tuned based on the choice of $\beta$.

   \item \textbf{Numerical accountant for arbitrary mechanisms}:   
     Previously, privacy accounting for new mechanisms required the noise distributions to be well-behaved (e.g., analytically derivable \Renyi Divergence values between distributions); our work develops a robust framework for using the PRV accountant \citep{microsoft_PRV_2} on new mechanisms that are not so well-behaved. This method many enables future research into novel DP mechanisms that cannot be directly analyzed analytically.
\end{enumerate}

Our results on the optimization of $\beta$ can be interpreted as a negative result: the Gaussian Mechanism ($\beta=2$) performs as well as or better than other values tested, and little to no additional improvements in performance can be gained by optimizing $\beta$. Our privacy accounting results can be interpreted as positive results, both for direct analysis of the GG-Mechanism, and may also be of independent interest for privacy accounting of new mechanisms.

%% file: sections_revised/02-background.tex
\subsection{Differential Privacy}\label{s.dpprelims}

 Differential privacy (DP) is a framework for designing privacy-preserving data analysis algorithms that protect the privacy of individuals in a dataset while allowing accurate statistical analysis. Informally, DP  provides a mathematical guarantee that an individual's data will have only a limited affect on the result of analysis on a large database. Two datasets are said to be \emph{neighboring} if they differ only in a single data record.
 
\begin{definition}[Differential privacy \citep{original_DP}]
      A mechanism $\mathcal{M}: \mathcal{D} \to \mathcal{R}$ satisfies $(\epsilon, \delta)$-differential privacy if for any two neighboring datasets $D, D' \in \mathcal{D}$ and for any $S \subseteq \mathcal{R}$,
      \[
            \Pr[\mathcal{M}(D)\in S ] \leq e^\epsilon \cdot \Pr[\mathcal{M}(D')\in S ] + \delta.
      \]
\end{definition}

Smaller values of the parameters $\epsilon$ and $\delta$ correspond to stronger privacy guarantees. The Laplace and Gaussian Mechanisms are examples of \emph{output perturbation} mechanisms, which first evaluate a function on the input dataset, and then add mean-zero noise to the result. The variance of the noise scales with the \emph{sensitivity} of the function $\Delta f$, defined as the maximum change in the function's value due to the removal or addition of a single database entry. 

\begin{definition}[$\ell_\beta$ sensitivity \citep{original_DP}] 
   The \(\ell_\beta\)-sensitivity of a function \(f \colon \mathbb{N}^{\lvert x \rvert} \rightarrow \mathbb{R}^k\) is:
\[
\Delta_\beta(f) = \max_{D,D' \text{ neighbors}} \lVert f(D) - f(D') \rVert_\beta.
\]
\end{definition}

We emphasize that while differently-normed sensitivity denotes that the function is bounded for a different $\ell_{\beta}$ norm, the definition of a neighboring dataset remains the same -- two datasets are neighboring if they differ by a single entry.

Two of the most common mechanisms in differential privacy are the \Laplace~Mechanism and the \Gaussian~Mechanism. 
\begin{definition}
    The \Laplace~Distribution (centered at 0) with scale $b$ is the distribution with probability density function:
    \[        Lap(x|b) = \frac{1}{2b} \exp ( - \frac{|x|}{b} ).
    \]
\end{definition}

\begin{definition}[\Laplace~Mechanism \citep{original_DP}]
    For any $\epsilon > 0$, given a real-valued function $f: \mathcal{D} \to \mathbbm{R}$, the \Laplace~mechanism is defined as
    \begin{equation*}
        \mathcal{M}_L(D, f, \epsilon) = f(D) + Y,
    \end{equation*}
    where $Y\sim Lap(\Delta f / \epsilon)$. The \Laplace~Mechanism is $(\epsilon,0)$-DP.
\end{definition}

\begin{definition}[\Gaussian~Mechanism \citep{privacy_book}]
    For any $\epsilon>0$ and $\delta \in (0,1]$, given a real-valued function $f: \mathcal{D} \to \mathbbm{R}$, the \Gaussian~mechanism is defined as
    \begin{equation*}
        \mathcal{M}_G(D, f, \epsilon) = f(D) + Y,
    \end{equation*}
    where $Y\sim \mathcal{N}(0,\sigma)$ for $\sigma > \Delta f \sqrt{2 \log (1.25 / \delta)}/\epsilon$. The \Gaussian~Mechanism is $(\epsilon,\delta)$-DP.   
    \end{definition}

One main feature of differential privacy is that the guarantees \emph{compose}, meaning  that the overall privacy loss (as measured by $\epsilon$ and $\delta$) of running multiple DP mechanisms can be bounded as a function of the privacy parameters of the individual mechanisms. In the simplest version of composition \citep{original_DP}, the privacy parameters ``add up'' so that running two $(\epsilon, \delta)$-DP mechanisms results in $(2 \epsilon, 2\delta)$-DP overall. In practice, however, this naive composition dramatically overestimates the incurred privacy risk, and more advanced composition algorithms \citep{advanced_composition} and privacy accountants \citep{dpsgd} are used to more accurately bound privacy risk. 

A privacy accountant is a tool to track the privacy budget of a system by recording the privacy cost associated with each query; accountants are particularly important for applications like \DPSGD, where DP mechanisms are composed a large number of times, e.g., as many as the number of steps in gradient descent. The introduction of the Moments Accountant \citep{dpsgd} enabled the first use of \DPSGD~with reasonable privacy guarantees on common datasets like MNIST \citep{mnist}. This was later replaced in many settings by accountants that rely on Renyi Differential Privacy (RDP), introduced by \citet{RDP_2017}.

\subsection{\Renyi Differential Privacy} \label{appendix:RDP}

\Renyi Differential Privacy (RDP) generalizes pure differential privacy $(\delta = 0)$ and is closely related to the moments accountant. Defined below, the RDP guarantee of a mechanism is stated in terms of \renyi divergence.
\begin{definition}[\renyi Divergence] The \renyi divergence of order $\alpha$ between two distributions $P$ and $Q$ is defined as:
    \begin{align*}
        D_\alpha(P||Q) =  \frac{1}{\alpha -1} \log \mathbbm{E}_{x \sim Q} \left[ (P(x)/ Q(x))^\alpha \right] = \frac{1}{\alpha -1 } \log \mathbbm{E}_{x \sim P} \left[ (P(x) / Q(x))^{\alpha -1} \right]. \\
    \end{align*}
\end{definition}

\begin{definition} \label{def:RDP-calculation}
    (\renyi Differential Privacy \citep{RDP_2017}). A randomized mechanism $\mathcal{M}$ satisfies $(\alpha, \epsilon)$-RDP with $\alpha \geq 1$ if for any neighboring datasets $D$ and $D'$:
    \begin{equation*}
        D_\alpha (\mathcal{M}(D)|| \mathcal{M}(D')) = \frac{1}{\alpha -1} \log \mathbbm{E}_{x \sim \mathcal{M}(D)} \left[ \left(\frac{\Pr[\mathcal{M}(D) =x ]}{\Pr[\mathcal{M}(D') =x ]} \right)^{\alpha -1} \right] \leq \epsilon.
    \end{equation*}

\end{definition}

RDP is desirable in ML applications because of its straightforward composition properties: the adaptive composition of mechanisms $\mathcal{M}_1, \ldots, \mathcal{M}_k$ where each $\mathcal{M}_i$ satisfies $(\alpha, \epsilon_i)$-RDP, will together satisfy $(\alpha, \sum_{i=1}^k \epsilon_i)$-RDP.

Pure $(\epsilon,0)$-DP corresponds to $(\infty, \epsilon)$-RDP; \cite{RDP_2017} provided more general guarantees for converting between RDP and DP: if $\mathcal{M}$ is an $(\alpha, \epsilon)$-RDP mechanism, it also satisfies $(\epsilon + \frac{\log (1/ \delta)}{\alpha -1}, \delta)$-DP for any $\delta \in (0,1)$.

\subsection{PRV Accountant} \label{sec:prv_accountant_background}

The privacy guarantees of a DP mechanism are commonly defined as a single tuple $(\epsilon, \delta)$, but can also be described as a function, since the probability of failure ($\delta$) can depend on the required $\epsilon$-bound. This naturally leads to the definition of a privacy curve, such that for every $\epsilon \in \mathbb{R}$, $\mathcal{M}$ is $(\epsilon, \delta(\epsilon))$-DP for the appropriate function $\delta(\epsilon)$. \citet{microsoft_PRV_2} provided an efficient method for composing privacy curves directly that gave much tighter privacy guarantees, using an accountant called the Privacy Random Variable accountant (PRV). This relies on a connection between a DP mechanism's privacy curve $\delta(\epsilon)$ and its uniquely defined privacy loss random variables ($X,Y$), which represent the likelihood of returning a particular outcome on two neighboring databases, respectively defined as:
\[
      X = \log(\tfrac{Q(\omega)}{P(\omega)}) \textrm{ where } \omega \sim P; \quad
      Y = \log(\tfrac{Q(\omega)}{P(\omega)}) \textrm{ where } \omega \sim Q,
\]
where \textit{P} and \textit{Q} are the distribution of the mechanism's output over two neighboring datasets. Intuitively, the privacy loss random variables can be thought of as the \textit{actual $\epsilon$} value for a specific output; it is a random variable because the output $\mathcal{M}(D)$ is itself a random function.

\citet{microsoft_PRV_2} introduced the algorithm ComposePRV, which efficiently computes the privacy guarantees for the composition of multiple DP mechanisms. ComposePRV takes as input the CDFs of PRVs $Y_1, \ldots, Y_k$ (as well as a few other hyperparameters), and returns an estimate of the privacy curve for all the mechanisms composed, represented by $\delta(\epsilon)$, allowing for the direct computation of $\epsilon$.

%% file: sections_revised/03-theory.tex
In this section, we first introduce the Generalized Gaussian (GG) Mechanism and show that it satisfies DP (Section \ref{s.mechs}). Since these privacy results are existential, rather than descriptive -- i.e., we show that \emph{there exists} some $\epsilon$ and $\delta$ values, rather than providing a closed form relationship between $(\beta,\sigma)$ and $(\epsilon,\delta)$ -- we also present a PRV-based privacy accounting method (Section \ref{s.accounting}) that can be used to measure explicit $(\epsilon,\delta)$-DP guarantees in the applications of this mechanism to ML tasks, as studied in Sections \ref{sec:argmax} and \ref{sec:dpsgd}.

\subsection{Generalized Gaussian Mechanism}\label{s.mechs}

We first formally define the Generalized Gaussian distribution and introduce the Generalized Gaussian Mechanism (Algorithm \ref{alg:GG_mechanism}), which is an output perturbation mechanism that adds noise sampled from the Generalized Gaussian distribution.

\begin{definition}[\citep{analytic_ggd}] \label{def.gengauss}
The Generalized Gaussian distribution, denoted $\mathcal{N}_\beta(\mu,\sigma)$, is specified by the pdf $p(x| \mu, \sigma, \beta) \propto e^{- \frac{|x- \mu|^\beta}{\sigma}}$ with normalizing constant $\frac{\beta}{2 \sigma^{\frac{1}{\beta}} \Gamma(\frac{1}{\beta})}$.
\end{definition}

The Generalized Gaussian Mechanism was introduced by \citep{Generalized_Gaussian_Mechanisms} for $\ell_1$ sensitivity and by \citep{GG_mechanism_1} for $\ell_\beta$ sensitivity; we use the latter, more general version. 

\begin{algorithm}[H]
    \caption{Generalized Gaussian Mechanism, $GG_{\beta, \sigma}(f,D)$.  \citep{GG_mechanism_1}}\label{alg:GG_mechanism}
    \begin{algorithmic}[1]
        \State \textbf{Input:} noise parameters $\beta \geq 1$, $\sigma >0$, vector-valued function $f: \mathcal{D} \to \mathbb{R}^d$, database $D \in \mathcal{D}$ 
        \State Let $\Delta_\beta f = \max_{D,D' \text{ neighbors}} \|f(D) - f(D')\|_\beta$
        \For{$i=1$ to $d$}
        \State Sample $Y_i \sim \mathcal{N}_{\beta}(0, \sigma \cdot \Delta_{\beta} f)$
        \EndFor
        \State {\bf Output:} $f(D) + (Y_1, \ldots Y_d)$ 
    \end{algorithmic}
\end{algorithm}

Next, we show that the GG Mechanism satisfies DP (Theorem \ref{thm.ggprivacy}). As stated, the result is existential, and does not provide an explicit relationship between the ($\beta,\sigma$) noise parameters of the mechanism and the resulting DP parameters ($\epsilon,\delta$). Critically, this guarantee that the mechanism is DP for \emph{some} $\epsilon$ and $\delta$ is sufficient to apply the PRV accountant described in Section \ref{s.accounting}.

\begin{restatable}{theorem}{ggprivacy}\label{thm.ggprivacy}
%\begin{theorem}\label{thm.ggprivacy}
    For any $\beta \geq 1$, $\sigma >0$, $\delta > 0$ there exists a finite value $\epsilon$ such that $GG_{\beta, \sigma}(\cdot,\cdot)$ satisfies $(\epsilon, \delta)$-DP.
%\end{theorem}
\end{restatable}

To prove Theorem \ref{thm.ggprivacy}, we first show that for all $\mu \geq 0$ (corresponding to the difference in the value of $f$ on neighboring databases) and all $\alpha>1$, the $\alpha$-\Renyi divergence between $\mathcal{N}_{\beta}(0, \sigma)$ and $\mathcal{N}_{\beta}(\mu, \sigma)$ is bounded by some finite $r$; this step is shown formally in Lemma \ref{lemma:renyi_div_is_bounded} in Appendix \ref{appendix:rdp_of_gg_is_bounded}. Bounded \Renyi divergence means that the $GG_{\beta, \sigma}(f,D)$ mechanism satisfies $(\alpha, r)$-RDP, and by the RDP-to-DP conversion of \citet{RDP_2017}, it also satisfies $(r + \frac{\log 1/ \delta}{\alpha -1}, \delta)$-DP for any $\delta \in (0,1)$. A formal proof of Theorem \ref{thm.ggprivacy} is given in Appendix \ref{appendix:gg_is_private}.

In \Cref{appendix:subsampling}, we also introduce and analyze the Sampled Generalized Gaussian Mechanism, SGG (\Cref{alg:SGG_mechanism}), which is a variant of the GG Mechanism that first applies Poisson subsampling to the input database, evaluates the function $f$ on the sample, and then adds Generalized Gaussian noise to the result. This mechanism is motivated by \emph{privacy amplification by subsampling}, which strengthens privacy guarantees without increasing the level of noise added by the mechanism, by subsampling the database before applying a DP mechanism. It is particularly popular in ML applications, such as the Sampled Gaussian Mechanism (SGM) \citep{SGM_rdp} and \DPSGD, because it formalizes the privacy gains intuitively brought by computing each update on the small subset of training examples selected to form each minibatch. 

\subsection{Privacy Accounting for GG Mechanisms}\label{s.accounting}

We focus on the PRV accountant because it is implemented in common codebases such as Opacus \citep{opacus_paper} and it empirically provides tighter guarantees than other accountants \citep{microsoft_PRV_2}. In this work we extend the PRV accountant to work for privacy accounting of arbitrary DP mechanisms such as the GG mechanism, which do not typically exist in closed-form, providing a framework for investigating alternative DP mechanisms. In order to calculate the privacy consumed using a PRV accountant, one must simply compute the CDF of the PRV. In \Cref{appendix:analytic_single_dim_PRV}, we extend the known PRVs for \Laplace~and \Gaussian~Mechanisms and compute a closed-form expression for the PRV of the Generalized Gaussian Mechanism, which enables us to apply the PRV accountant.

We compute the privacy guarantees of $GG_{\beta,\sigma}(f,D)$ using the PRV accountant by sampling from the appropriate $\mathcal{N}_\beta(\mu, \sigma)$ distribution, and numerically computing an empirical estimate of the CDF of $Y$. \citep{microsoft_PRV_2} provides a tight estimate of the error in the PRV accountant's estimate; our work relies on this computation and provides a similarly tight error analysis that bounds the additional error from sampling, which is presented in~\Cref{appendix:privacy_in_practice}. We then plug in this CDF as input to the ComposePRV algorithm of \citet{microsoft_PRV_2}, which takes the CDFs of the PRVs of the composed mechanisms, and returns a composed privacy curve $\delta(\epsilon)$, providing the unique $\epsilon$ value for a specified choice of $\delta$.  We call this method the \textit{Sampled PRV accountant}; the implementation of the PRV accountant and Sampled PRV accountant is described in greater detail in~\Cref{appendix:privacy_in_practice}.

Figure \ref{fig:epsilon_dp_as_a_function_of_beta} illustrates the resulting value of $\epsilon$ as a function of $\sigma$ for different values of $\beta$ and fixed values of $\delta=1e-5$ and $\Delta f =1$.  Observe that all curves have a similar structural form to the known \Gaussian\ ($\beta=2$) and \Laplace\ ($\beta=1$) Mechanisms, and that as $\beta$ grows, the same $(\epsilon, \delta)$-DP guarantee necessitates a larger value of $\sigma$. In the figure, we show the privacy curves for a single composition of the GG mechanism, but importantly, these privacy curves and their relative differences change as a function of the number of times the mechanisms are composed. 

\begin{figure}[ht]
    \centering
    \includegraphics[width=1.0\textwidth]{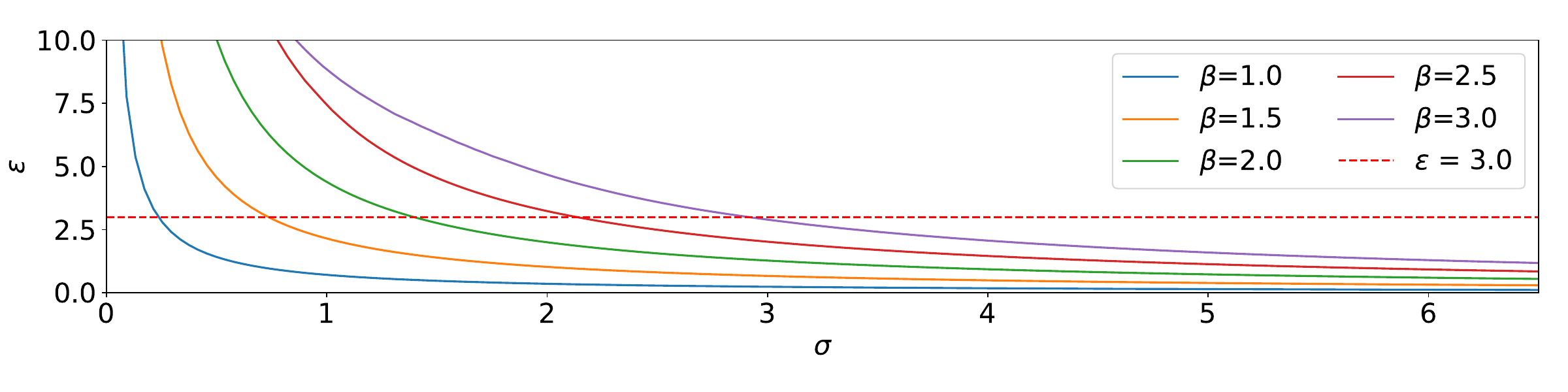}    
    \caption{\small{
    DP parameter $\epsilon$ as a function of noise parameter $\sigma$ with fixed $\delta = 10^{-5}$ and $\Delta f=1$, calculated using the PRV accountant. Mechanisms with equivalent DP guarantees can be identified by computing the privacy curve's intersection with a horizontal line, illustrated here with a red line for  (arbitrarily chosen) $\epsilon=4$. }}
    \label{fig:epsilon_dp_as_a_function_of_beta}
\end{figure}

When using sensitivity defined in the $\ell_{\beta}$ norm we observe that the privacy cost of the multi-dimensional mechanism is equivalent to the single dimension mechanism when $\beta \leq 2$. This makes privacy accounting for $GG_{\beta,\sigma}(f,D)$ dimension independent for $\beta$ in this range. We prove this in result \Cref{appendix:dim_independence_GG}, and also provide an analytic solution for the one dimensional GG mechanism in~\Cref{appendix:analytic_single_dim_PRV}. This dimension-independence is crucial for machine learning, as it allows us consider numerical, sampling-based approaches. Getting a sufficiently accurate empirical CDF requires sampling millions of times from the PRV, which entails sampling from the noise distribution used in the privacy mechanism distribution. Thus if privacy accounting depended on dimension, sampling numerically would be computationally inefficient in high dimensions. Even very small machine learning models typically have over 1M parameters (dimensions), requiring us to sample from 1M dimensional distributions millions of times.

For a given privacy budget $(\epsilon, \delta)$, varying $\beta$ will change the paired $\sigma$ and subsequently vary the weight of the tails of the distributions. As a result, for a particular definition of ``tail'', it is possible to derive the optimal GG mechanism that minimizes the likelihood of an outlier; we explore this in \Cref{section:outliers_for_equivalent_mechanisms} and believe may be of independent interest.

%% file: sections_revised/04-argmax.tex
Private Aggregation of Teacher Ensembles (PATE) \citep{pate_2017} is an algorithm for training a private machine learning model. In PATE, a dataset is partitioned and a model is trained on each partition, then the models in the ensemble privately vote on the labels of an unlabeled dataset, and finally a model is trained on the privately labeled dataset. A more detailed explanation is presented in \Cref{appendix:argmax_and_pate}. The core step in PATE that provides formal privacy is the private computation of the \Argmax~over the votes of an ensemble of models, returning the plurality of a noisy histogram of votes. This is done based on a variation of the ReportNoisyMax algorithm, and is the center of our focus in this section.

 PATE has been extensively studied with variations of \Laplace~and \Gaussian~Mechanisms; \citet{pate_2017} employed LNMax, which privately aggregates votes from an ensemble of models by taking the argmax of a histogram after adding \Laplace~noise. Later, \citet{pate_2018} developed several variations based on adding \Gaussian~noise, including GNMax, and empirically found them to be superior to their \Laplace~counterpart. 
 
 We introduce a new algorithm, GGNMax, generalizing the GNMax and LNMax algorithms, which adds noise from the Generalized Gaussian Distribution $\mathcal{N}_\beta(0,\sigma)$ (Section \ref{s.privargmax}). We then show empirically in Section \ref{sec:argmax-experiments} that the effect of $\beta$ on average label accuracy is relatively weak and \Laplace~and \Gaussian~noise work produce nearly equivalent privacy-accuracy tradeoffs. We supplement these findings with simulations in \Cref{sec:argmax_simulations}, and we empirically show that the \Gaussian~mechanism is near-optimal when the correct label of the histogram aligns with the non-private majority vote.

\subsection{Private \Argmax~and the GGNMax Mechanism}\label{s.privargmax}

We present our Generalized Gaussian Private \Argmax~algorithm (GGNMax), which takes in noise parameters ($\beta, \sigma$), a set of count functions each with sensitivity $1$, and a database. The algorithm adds noise sampled from $\mathcal{N}_{\beta}(0, \sigma)$ to each function value and then returns the index of the function with the largest noisy value. This will be used as a key private subroutine of PATE, where the count functions will be used to privately count the number of votes for each label.

\begin{algorithm}[H]
    \caption{Generalized Gaussian Private \Argmax, GGNMax($\beta,\sigma,\{f\},D$) }\label{alg:beta_private_argmax}
    \begin{algorithmic}[1]
        \State \textbf{Input} noise parameters $\beta \geq 1$, $\sigma>0$, count functions $f_1,\ldots, f_N : \mathcal{D} \to \mathbb{R}$, database $D \in \mathcal{D}$
        \For{$i=1$ to $N$}
        \State Compute $f_i(D)$
        \State Sample $Y_i \sim \mathcal{N}_\beta(0,\sigma)$
        \EndFor
        \State {\bf Output} $\arg\max_{i\in[N]}\{f_i(D) + Y_i\}$
    \end{algorithmic}
\end{algorithm}

The privacy guarantees of GGNMax rely on the privacy of the GG Mechanism, which is used as a subroutine. Note that the GG Mechanism is known to be private for \emph{some} ($\epsilon,\delta$) pair (Theorem \ref{thm.ggprivacy}), so this theorem relates the parameters of the two mechanisms.

\begin{restatable}{theorem}{ggnmaxIsPrivate} \label{thm.ggargmaxpriv}
 %   \begin{theorem}\label{thm.ggargmaxpriv}
        If the ($\beta, \sigma$)-Generalized Gaussian Mechanism is $(\epsilon, \delta)$-DP for a fixed $\epsilon>0$ and $\delta\geq 0$, then ($\beta, \sigma$)-Generalized Gaussian Private \Argmax\ is also $(\epsilon, \delta)$-DP.
%    \end{theorem}
\end{restatable}

A proof of Theorem \ref{thm.ggargmaxpriv} is given in Appendix \ref{appendix:private_argmax_is_private}. Note that as with the \Laplace-based Report Noisy Max algorithm \citep{privacy_book} and the \Gaussian-based GNMax algorithm \citep{pate_2018}, the privacy guarantee does not depend on number of functions $N$. This makes it a desirable method for computing the argmax of the output of the GG Mechanism, which consumes privacy proportional to the number of queries; this property allows GGNMax to have privacy-accuracy tradeoffs that scale well in high dimensions.

\subsection{PATE Experiments} \label{sec:argmax-experiments}

As an intermediate step in PATE, the algorithm produces a histogram of votes in order to privately label an unlabeled dataset. In order to isolate the effect of the specific privacy mechanism used for PATE, we focus our attention to study the effect of the $\beta$ parameter on the label accuracy of the private aggregation step in PATE. 

To study this for a realistic setting, we use the histograms generated on the MNIST and the Street View House Numbers (SVHN) dataset \citep{svhn}, produced by \citet{pate_2017}, which introduced the PATE algorithm. For each of these datasets, we start with a collection of $10,000$ histograms; each histogram is the collection of 250 models trained on a partition of the dataset, and evaluated on an unlabeled datapoint $x$. Then, for each histogram, we compute the private label produced by the GGNMax mechanism for 20 evenly spaced values $\beta \in [1,4]$ and 200 values of $\sigma \in [0.01, 7]$. For each fixed $(\epsilon, \delta)$ and value of $\beta$, we compute the average label accuracy with respect to the ground truth labels provided by the dataset, averaged across 25 trials for each datapoint.\footnote{We find that changing the number of trials has minimal effect on the mean or standard deviation of the average accuracy} 

\begin{figure}[tbh]
    \centering
    \subfigure{\includegraphics[width=0.95\textwidth]{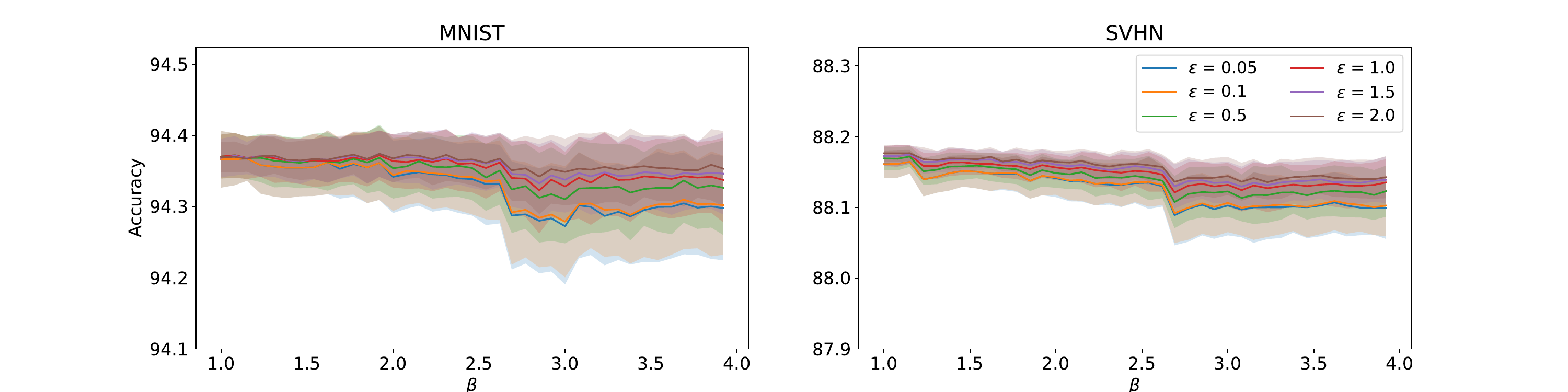}}
    \caption{\small{Average Label Accuracy of GGNMax mechanisms with equivalent privacy guarantees and varying values of $\beta$, evaluated on histograms which were generated by \citet{pate_2017} as an intermediate state produced by 250 teachers trained on MNIST (left) and SVHN (right)}. 
    }
    \label{fig:pate_argmax_accuracy}
\end{figure}

Figure \ref{fig:pate_argmax_accuracy} shows the average label accuracy for the GGNMax mechanism applied to histograms generated as part of PATE for MNIST and SVHN, as a function of $\beta$. We observe that values of $\beta$ in the region of $\beta \in [1,2.5]$ perform roughly equivalently, and for larger choices of $\beta$, the label accuracy decreases slightly.\footnote{Empirically, we observe a small drop in performance around $\beta=2.6$, which may be due to an artifact of how the mechanisms of equivalent DP guarantees are generated. First we add noise from an evenly distributed grid of $\beta \in [1,4]$ and $\sigma \in [0.01, 7]$ values, and then we compute the privacy for those $(\beta,\sigma)$ tuples and compute the corresponding mechanisms of approximately equal DP guarantees. This may cause us to overestimate the $\epsilon$ when solving for mechanisms with equivalent DP guarantees.}

To further extend these findings to more settings, we investigate the privacy-accuracy tradeoff of the GGNMax algorithm on simulated histograms in \Cref{sec:argmax_simulations}. We find that when the correct label is the argmax of the unnoised histograms, the GGNMax with $\beta$ values close to 2 perform near-optimally.

%% file: sections_revised/05-dpsgd.tex
We now turn to applications of the GG Mechanism in \DPSGD, which is one of the most commonly used mechanisms for private ML. We propose two simple changes to \DPSGD: replacing the Gaussian noise used in \DPSGD~with Generalized Gaussian noise, and using the $\ell_{\beta}$ norm for clipping rather than $\ell_2$.\footnote{We choose to use $\ell_{\beta}$ clipping rather than a fixed choice like $\ell_2$ because when using the $GG_{\beta, \sigma}(f,D)$ mechanism with $\ell_{\beta}$ sensitivity, privacy accounting is dimension-independent. See \Cref{appendix:dim_independence_GG}.} As articulated in Section \ref{s.dpprelims}, this change in sensitivity measure does not change the notion of neighboring databases, so optimization over $\beta$ shoulve be viewed as additional hyperparameter tuning to improve performance, rather than changing the nature of the privacy guarantees.

We call the resulting mechanism the $\beta$-Generalized Gaussian Differentially Private SGD ($\beta$-DP-SGD). This is also equivalent to changing the underlying mechanism in \DPSGD~from the Sampled Gaussian Mechanism to the Sampled Generalized Gaussian Mechanism (SGG), as \DPSGD~performs Poisson subsampling on the dataset before computing the gradient update and adding Gaussian noise. The $\beta$-DP-SGD algorithm is presented formally in Algorithm \ref{alg:beta_dpsgd}.\footnote{Note that when using a privacy accountant with a desired $(\epsilon,\delta)$-DP guarantee, the privacy accountant tracks the privacy used at each epoch and halts when the privacy budget has been reached, rather than running for a fixed training length $T$ as in Algorithm \ref{alg:beta_dpsgd}. We use the current presentation for simplicity.}  

\begin{algorithm}[tbh]
    \caption{$\beta$-Generalized Gaussian Differentially Private SGD, $\beta$-DP-SGD($\beta, \sigma, D,l,\eta,L,C,T$) }\label{alg:beta_dpsgd}
    \begin{algorithmic}[1]
        \State \textbf{Input:} noise parameters $\beta\geq 1$, $\sigma>0$, database $D =\{x_1,\ldots,x_N\}$ of points in $\mathbb{R}^d$, loss function $l(\theta , x_i)$, learning rate $\eta$, average group size $L$, clip norm $C$, and training epoch length $T$
        \State Initialize $\theta_1 \in \mathbb{R}^d$ randomly 
        \For{$t =1$ to $T$}
        \State Construct $L_t \subseteq D$ such that each $x_i \in D$ is included with probability $q = L/|D|$ (Poisson sampling)
        %\STATE {\bf Compute Gradient}
        \For{each $i\in L_t$}
        \State {Compute $G_t(x_i) = \nabla_{\theta_t} l(\theta_t, x_i)$}
        %\State {\bf Clip Gradient} 
        \State {$\bar{G}_t(x_i) \gets G_t(x_i) / \max\big(1, \frac{\|G_t(x_i)\|_{\beta}}{C}\big)$}
        \EndFor
        %\State {\bf Noise addition}
        \State Sample $Y_1,\ldots,Y_d \sim_{i.i.d.} \mathcal{N}_{\beta}(0, \sigma \cdot C)$ 
        \State $\tilde{G}_t \gets \frac{1}{L}\left( \sum_i \bar{G}_t(x_i) + \vec{Y}\right)$ 
        %\STATE {\bf Model Update}
        \State { $\theta_{t+1} \gets \theta_{t} - \eta \tilde{G}_t$}
        \EndFor
        \State {\bf Output:} $\theta_T$
    \end{algorithmic}
\end{algorithm}

The following theorem states that $\beta$-\DPSGD~is differentially private, which we prove in \Cref{proof:dpsgd_is_dp}.

\begin{restatable}{theorem}{DPSGDIsDP}\label{thm:DPSGD_is_DP}
For any $\delta>0$, $\beta \geq 1$, $\sigma >0$, $f: \mathcal{D} \rightarrow \mathbb{R}^d$, database $D \in \mathcal{D}$, for any loss function of the form $l(\theta, x_i)$, learning rate $\eta \geq 0$, average group size $L$, clipping norm $C \geq 0$, there exists a finite $\epsilon > 0$ such that the algorithm $\beta$-\DPSGD$(\beta, \sigma, D, l, \eta, L,C)$ satisfies $(\epsilon, \delta)$-DP.
\end{restatable}

\subsection{\DPSGD~Experiments} \label{sec:dpsgd_experiments}

We seek to find a relationship between $\beta$ and \DPSGD's privacy-accuracy trade-off for non-convex optimization tasks by comparing test-accuracy as a function of $\epsilon$, for different $\beta$. 

To provide a robust evaluation of the role of $\beta$ in the $\beta$-\DPSGD\ algorithm, we focus on 4 datasets in different domains: CIFAR-10 \citep{cifar_10} and Street View House Numbers (SVHN) \citep{svhn}, two common computer vision datasets; the Adult dataset \citep{adult_dataset}, a tabular dataset with a binary classification task; and the IMDB dataset \citep{imdb_dataset}, a collection of movie reviews meant for binary sentiment classification. We train three different architectures: for the vision classification tasks, we use the models described in \cite{handcrafted_dp}, which previously achieved SOTA results for the $\epsilon \leq \sim 2.5$ regime (ScatterNet CNNs). For the the Adult Dataset we train a 2-layer Fully Connected Network (FCN), and for the IMDB dataset, we train a Long-Short Term Memory (LSTM) network with $\sim$ 1M parameters.

A full description of the hyperparameters, datasets, and models is included in \Cref{appendix:hyperparameters_in_dpsgd}. Each experiment is run 3 times, which we found sufficient given standard deviations that generally fell below 0.3\%.\footnote{Although 3 iterations would typically be considered a small number, each run of the mechanism is already aggregating over $T$ rounds of the GG mechanism and corresponding gradient updates.} We find that when fixing a choice of $\beta$ and allowing for hyperparameter tuning along all other hyperparameters, we see a weak but noticeable relationship with final test accuracy. We report the maximum test-accuracy for each dataset, i.e., the maximum across all hyperparameters under the fixed architecture, and the standard deviation across 3 trials for each set of hyperparameters.

We restrict to $\beta \in [1,2]$ because this is the range for which the GG Mechanism is \emph{dimension-independent}, which enables computational efficiency of our experiments. Specifically, to compute the privacy accounting for each of the ML models, we use the Sampled PRV accountant of Section \ref{s.accounting} (with additional details in \Cref{appendix:privacy_in_practice}). The Sampled PRV accountant relies on generating an empirical CDF by sampling a large number of times from the PRV of the mechanism, which entails sampling from the noise distribution used in the private mechanism. The dimension-independent privacy accounting property means that when sensitivity is measured in the $\ell_\beta$ norm, the PRV for a $d$-dimensional GG mechanism is identical to the PRV for the corresponding $1$-dimensional mechanism. Thus running the Sampled PRV accountant only requires drawing samples from a $1$-dimensional GG distribution. This is a critical computational saving: 
for neural networks with $d > 10^6$ parameters, sampling from a $d$-dimensional distribution millions of times would be computationally infeasible. In Appendix \ref{appendix:dim_independence_GG}, we prove that for $\beta \leq 2$, privacy accounting is dimension independent.

\Cref{fig:dpsgd_results_scatter__horizontal} presents the maximum test-accuracy for 3 different values of $\epsilon$, evaluated for 3 different models on 4 different datasets (ScatterNet on CIFAR-10 and SVHN, a FCN on the Adult dataset, and an LSTM on the IMDB dataset), across varying values of $\beta \in [1,2]$.

In this regime of $\beta \in [1,2]$, we find that $\beta = 2$ empirically outperforms other choices of $\beta$. 
While our work is able to explain why the Gaussian mechanism performs well in practice, computational challenges prevent us from drawing more general conclusions on its optimality across a wider range of values of $\beta$ ($\beta >2$). Extending our empirical results to the $\beta>2$ setting would require further advances to privacy accounting, which we leave to future work.

In \Cref{app:hyperparams_in_dpsgd}, we explore the relationship of individual hyperparameters with $\beta$ and find a noticeable effect of $\beta$ on the final test-accuracy, particularly for smaller $\epsilon$.

\begin{figure}[htb]
    \centering
    \includegraphics[width=0.8\textwidth]{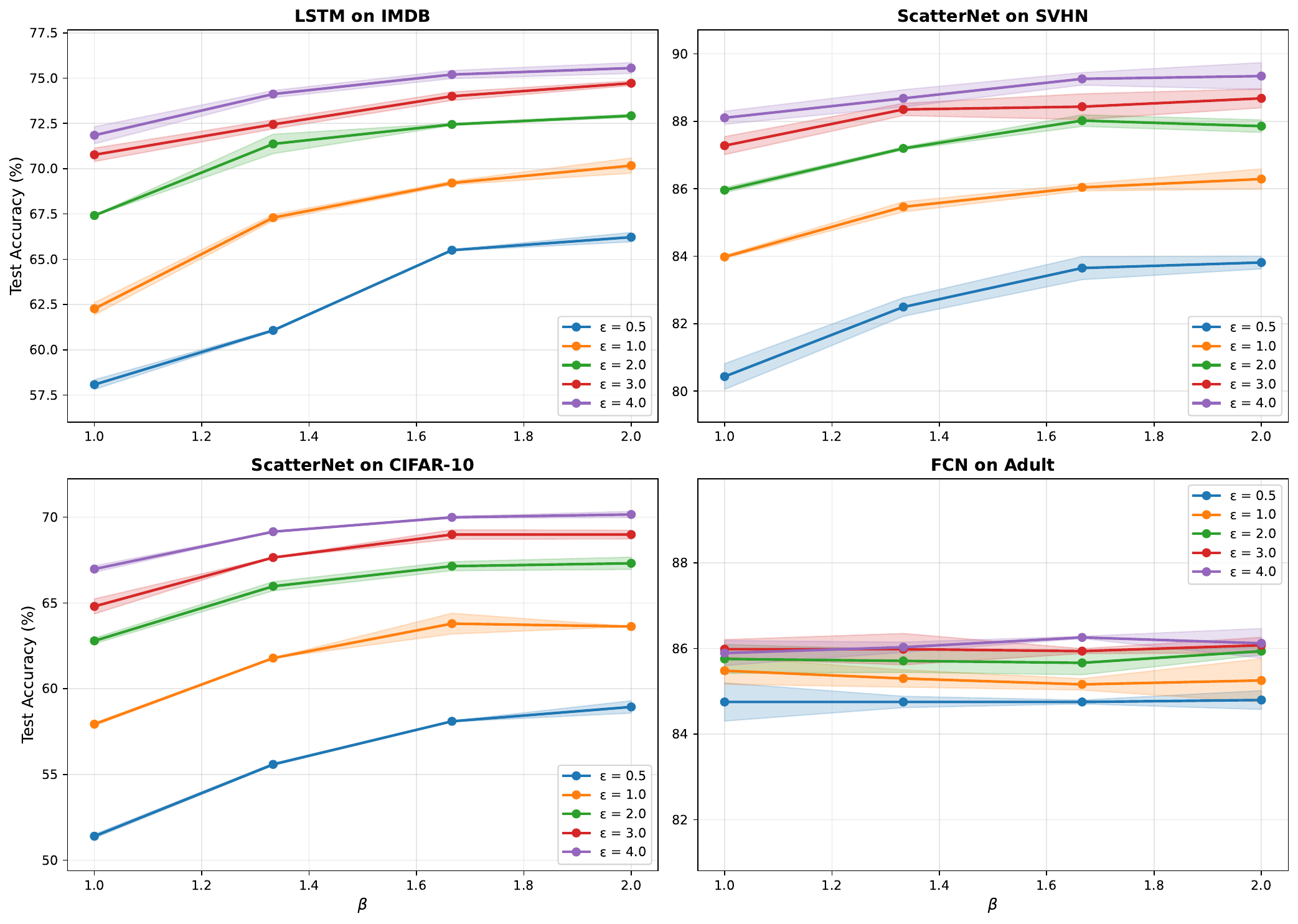}
    \caption{\small{$\beta$-\DPSGD~results for the corresponding architectures trained on CIFAR-10, SVHN, Adult, and IMDB, for $\delta = 10^{-6}$. The test-accuracy is reported for five values of $\epsilon$, computed for each architecture-dataset pair. A vertical dashed line denotes the \Gaussian~mechanism. Note: Some values are not presented (for lower $\epsilon$), because larger $\beta$ tends to consume more privacy per step, and the model's privacy budget exceeds the target in the first step.}}
    \label{fig:dpsgd_results_scatter__horizontal}
\end{figure}

\Cref{tab:dpsgdresults} provides numerical results on the performance of $\beta$-\DPSGD~compared to that of regular \DPSGD~across different $\epsilon$ values for the example case of CIFAR-10. We observe that test-accuracy for $\beta=2$ is as-good or better that other values of $\beta$ we explored in $\beta \in [1,2]$. Most existing SOTA results use privacy guarantees provided by an RDP accountant, which can overestimate privacy loss relative to PRV accounting. In order to disambiguate empirical differences due to improved accounting versus the GG mechanism, we recreate existing SOTA results using both the PRV and RDP accountants. As expected, we see that methods using the PRV accountant outperform those using the RDP accountant.

The comparison of $\beta$-\DPSGD~versus \DPSGD\ in \Cref{tab:dpsgdresults} reveals that within the range $\beta \in [1,2]$, we found that $\beta=2$ (Gaussian) performs as well as or better than other values. This implies that smaller choices of $\beta<2$ are likely not optimal for this task, but it is possible for a larger $\beta>2$ to perform even better.

\begin{table}[tbh]
    \centering
     \begin{tabular}{cllrrrr}
        \toprule
         Accountant & Training Algorithm & $\epsilon =1$ & $\epsilon =2$ & $\epsilon =3$ & $\beta$ value \\
        \midrule
        RDP& \DPSGD &60.3 (-)& 67.2 (-) & 69.3 (-) & - \\ 
        \hline 
        PRV & \DPSGD & 63.5 (0.4) & 67.6 (0.1) & 69.2 (0.3) & - \\
        PRV & $\beta$-\DPSGD~(\textit{ours}) &\textbf{63.7} (0.1) & \textbf{67.7} (0.1) & \textbf{69.4} (0.1) & (2, 2, 2) \\
        \bottomrule
    \end{tabular}
    \caption{\small{SOTA Results for private ML, evaluated on CIFAR-10, grouped by privacy accountant. The model used is a CNN model ($\sim$ 5e5 parameters) trained on Scatternet features, which we refer to as ScatterNet. The $\beta$ value column is set to `-' if trained with traditional \DPSGD, otherwise it reports a tuple of $\beta$ values that achieve max accuracy for $\epsilon =1,2,3$, respectively (ties broken by smaller std values). Best performance for each $\epsilon$ is bolded.
    } }
    \label{tab:dpsgdresults}    
\end{table}

%% file: sections_revised/appendix.tex
\section{Sampled Generalized Gaussian Mechanism} \label{appendix:subsampling}

\emph{Privacy amplification by subsampling} is a technique to strengthen DP guarantees without increasing the level of noise, by randomly sampling a subset of the input dataset before applying a DP mechanism; it is commonly used in ML applications. The DP parameters improve proportionally to the subsampling rate, as formalized in Theorem \ref{thm:privacy_amplification}. Intuitively, each point is less likely to be used in the analysis, and the noise from sampling can be ``counted'' toward the privacy budget.

For our mechanism, we will focus on \emph{Poisson subsampling}, a sampling process where each element of a population is included in a set according to the outcome of an independent Bernoulli trial. We use $S(q)$ to refer to a Poisson sampling procedure with sampling rate $q$.

\begin{theorem}[Privacy Amplification by Poisson Subsampling \citep{privacy_amplification_theorem_1} \citep{privacy_amplification_theorem_2}] \label{thm:privacy_amplification}
Let $\mathcal{M}$ be an $(\epsilon,\delta)$-DP mechanism, and let $\mathcal{S}(q)$ be a Poisson sampling procedure with sampling rate $q$. Then $\mathcal{M} \circ \mathcal{S}(q)$ is ($O(\log(q)\epsilon),q\delta$)-DP.
\end{theorem}

\subsection{Sampled Generalized Gaussian Mechanism} \label{sec:SGG_def }
Next, we present the Sampled Generalized Gaussian Mechanism, SGG, which is a sampled variant of the GG Mechanism. It generalizes the Sampled Gaussian Mechanism \citep{SGM_rdp}, which is a common mechanism in private ML. This mechanism relies on privacy amplification by subsampling (Theorem \ref{thm:privacy_amplification}) to attain improved privacy guarantees relative to its non-sampled counterpart. We state the SGG Mechanism in terms of Poisson sampling because the PRV accountant is defined only for Poisson sampling; the mechanism can immediately be extended other types of sampling and the privacy guarantees would still hold under the appropriate accountant.

\begin{algorithm}[H]
    \caption{Sampled Generalized Gaussian Mechanism, $SGG_{\beta, \sigma,q}(f,D)$}\label{alg:SGG_mechanism}
    \begin{algorithmic}[1]
        \State \textbf{Input} noise parameters $\beta \geq 1$, $\sigma >0$, sample rate $q \in (0,1]$ a vector-valued function $f: \mathcal{D} \to \mathbb{R}^d$, database $D \in \mathcal{D}$
        \State Compute $l_{\beta}$ sensitivity of $f$: $\Delta_{\beta} f = \max_{D,D' \text{ neighbors}} \|f(D) - f(D')\|_{\beta}$
        \State $S = \emptyset$
        \For{each data element $x_j \in D$}
        \State With probability $q$, add $x_j$ to $S$
        \EndFor
        \For{$i=1$ to $d$}
        \State Sample $Y_i \sim \mathcal{N}_{\beta}(0, \sigma \cdot \Delta_{\beta} f)$
        \EndFor
        \State {\bf Output} $f(S) + (Y_1, \ldots Y_d)$ 
    \end{algorithmic}
\end{algorithm}

The privacy guarantees of the SGG Mechanism follow nearly immediately from privacy of the GG Mechanism (Theorem \ref{thm.ggprivacy}) and privacy amplification by sampling (Theorem \ref{thm:privacy_amplification}). 

\begin{theorem}\label{thm.sggprivacy}
    For any $\beta \geq 1$, $\sigma >0$, $\delta > 0$, $q \in (0,1]$, there exists a value $\epsilon$ such that $SGG_{\beta, \sigma,q}(\cdot,\cdot)$ satisfies $(\epsilon, \delta)$-DP.
\end{theorem}
\begin{proof}
Theorem \ref{thm.ggprivacy} states that for any $\beta \geq 1$, $\sigma >0$, $\delta > 0$, there exists an $\epsilon>0$ for which $GG_{\beta, \sigma,q}(\cdot,\cdot)$ is $(\eps,\delta)$-DP. The SGG Mechanism simply applies Poisson Subsampling before running this GG Mechanism on the subsampled dataset, so by Theorem \ref{thm:privacy_amplification}, SGG satisfies ($O(\log(q)\epsilon),q\delta$)-DP for the same $\eps$.
\end{proof}

\subsection{$\beta$-\DPSGD~is Differentially Private}
The $\beta$-\DPSGD\ algorithm presented in \Cref{alg:beta_dpsgd} also relies on privacy amplification by subsampling for its privacy guarantees. At each time $t$, the algorithm performs Poisson subsampling with $q=L/|D|$ before its gradient measurement and update steps. We first provide the existential privacy result that $\beta$-\DPSGD\ is $(\eps,\delta)$-DP for \emph{some} privacy parameters (Theorem \ref{thm:DPSGD_is_DP}), and then provide Corollary \ref{corollary:DPSGD_privacy_bounds}, which rephrases the privacy guarantee in a way that enables use of the PRV accountant across all $T$ rounds.

\DPSGDIsDP*
\begin{proof} \label{proof:dpsgd_is_dp} 
\Cref{alg:beta_dpsgd} can be viewed as an adaptive $T$-fold composition of the SGG Mechanism, with post-processing on its results. Specifically, the algorithm's final output $\theta_T$ can be written as $\theta_T = \theta_0 + \eta \sum_{t}^T \tilde{G_t}$, where the $\tilde{G_t}$ is the postprocessed output of an instantiation of SGG with $q=L/|D|$ and function $f(D) = \sum_i \nabla_{\theta_t} l(\theta_t, x_i) / \max\big(1, \frac{\|\nabla_{\theta_t} l(\theta_t, x_i)\|_\beta}{C}\big)$ and noise vector $\vec{Y}$ sampled with each entry $Y_i \sim \mathcal{N}_\beta(0,\sigma C)$. Since -- given $\beta, \delta, \sigma, q$ -- SGG is ($\eps,\delta$)-DP for some $\epsilon$ (Theorem \ref{thm.sggprivacy}), then the adaptive composition and postprocessing of these mechanisms' outputs that is performed in \Cref{alg:beta_dpsgd} will also be $(\eps,\delta)$-DP for some $\epsilon$.
    \end{proof}

Corollary \ref{corollary:DPSGD_privacy_bounds} restates this result in a way that directly allows the PRV accountant to be applied for tighter composition across all $T$ rounds.

    \begin{restatable}{corollary}{DPSGDPrivacyBounds}\label{corollary:DPSGD_privacy_bounds}
        If the $SGG_{\beta, \sigma, q}(f,D)$ mechanism composed $T$ times on function $f$ with sensitivity $\Delta f$ satisfies $(\epsilon, \delta)$-DP, then for any $L\leq |D|$, $C = \Delta f$, and loss function $l(\theta, x_i)$, the $\beta$-\DPSGD$(\beta,\sigma, D, l, \eta, L, C)$ also satisfies $(\epsilon, \delta)$-DP.
    \end{restatable}  
    
    \begin{proof} \label{proof:dpsgd_privacy_bounds}
        As shown in the proof of \Cref{thm:DPSGD_is_DP}, the $\beta$-\DPSGD\ mechanism consists of the $T$-fold adaptive composition (with postprocessing) of the SGG Mechanism, with the clipping norm $C$ specified by the sensitivity $\Delta f$ of the input function $f$, so that $C = \Delta f$. Thus by Theorem \ref{thm:DPSGD_is_DP} and the post-processing and composition guarantees of DP, the $\beta$-\DPSGD\ mechanism also satisfies $(\eps,\delta)$-DP for the same $(\eps,\delta)$.
    \end{proof}

\section{Empirically Computing Privacy Guarantees} 

Due to the complexity that arises from composing multiple DP mechanisms together, any statement about the privacy-accuracy tradeoff of a highly-composed mechanism (like \DPSGD) is typically caveated by the specific implementation of the privacy accountant used for privacy accounting. We demonstrate an example of this in Table~\ref{tab:dpsgdresults} where we present a clear improvement in the privacy-accuracy tradeoff of existing methods by simply changing RDP to PRV accounting.

At the time of writing, numerical accountants such as the PRV accountant achieve the tightest privacy guarantees \cite{microsoft_PRV_2}. However, as a result of using a numerical accountant, closed form solutions for the privacy consumed do not typically exist. 
Specifically, \citet{microsoft_PRV_2} introduced the algorithm ComposePRV (presented here in Appendix \ref{appendix:privacy_in_practice}), which efficiently computes privacy guarantees for the composition of multiple DP mechanisms based on the PRV accountant. It takes as input the CDFs of PRVs $Y_1, \ldots, Y_k$, a mesh size $h$, and truncation parameter $L$, and returns an estimate of the privacy curve for all the mechanisms composed, represented by $\delta(\epsilon)$, enabling the direct computation of $\epsilon$. Given ComposePRV, \citep{microsoft_PRV_2} also shows that the PRV accountant can be directly used to compute a mechanism's privacy loss, including for Poisson subsampled variants of mechanisms such as SGG. 

However, when using the PRV accountant one must compute the CDF of the PRV, which is not a simple task for all differentially private mechanisms. In this work, we get around this by estimating the CDF of the PRV numerically; while this does introduce error, it also makes the privacy accounting of arbitrary mechanisms possible. In this appendix we show how to use the PRV accountant for the $\beta$-\DPSGD~mechanism by replacing the true CDF with an empirical estimate, and provide error bounds for the resulting privacy guarantee. In short, we replace the CDF of the PRV with the CDF of a (binned) histogram of the PRV sampled $n$ times, and account for the error introduced by using numerically computed histograms. 

In this appendix, we first show that how to compute the PRVs for the single-dimensional GG mechanism, and the multi-dimensional GG mechanism, and show that the GG mechanism is dimension-independent for specific sensitivity bounds (\Cref{appendix:analytic_single_dim_PRV} and \Cref{appendix:dim_independence_GG}). We then show how to use these PRVs to do privacy accounting for the GG mechanism and the sampled GG mechanism, and provide error bounds under the modified PRV accountant that estimates the CDF (\Cref{{appendix:privacy_in_practice}}). Lastly, we explore how to find mechanisms that satisfy equivalent $(\epsilon, \delta)$-DP guarantees, which can be used to define a new search space of comparable mechanisms (\Cref{sec:mechanisms_with_equal_privacy}); we end by proposing a niche, but potentially valuable use-case where the Generalized Gaussian for $\beta \notin \{1,2\}$ surpasses the Laplace or Gaussian Mechanisms, for the goal of preventing outliers (\Cref{section:outliers_for_equivalent_mechanisms}).

In this paper, our privacy accounting for $\beta$-\DPSGD~depends on numerically computing the CDF of the PRV for the Generalized Gaussian Mechanism by sampling from the Generalized Gaussian. For most distributions, computing the PRVs for a multi-dimensional output perturbation DP mechanism (e.g., a private mean in multiple dimensions) would depend on dimension. In typical machine learning applications (the primary use-case of \DPSGD), even small models regularly exceed 1 million parameters, so dependence on dimension could be catastrophic if estimating a histogram involved sampling from $d > 1,000,000$ distributions. Instead, we find that the PRVs for the Generalized Gaussian Mechanism do not depend on dimension when noise is sampled from $\mathcal{N}_\beta$ and the function has $|| \cdot||_{\beta}$-sensitivity. We note that this retroactively provides some rationale for why the Laplace and Gaussian Mechanism respectively use $\ell_1$- and $\ell_2$-norms for their noise distributions.

\subsection{Analytic PRV for a Single-Dimensional Generalized Gaussian Mechanism} \label{appendix:analytic_single_dim_PRV}

In order to compute the PRV for $GG_{\beta,\sigma}(f,D)$ we consider the privacy loss random variables for two distributions shifted by $\mu = \Delta f$, corresponding to the outputs of the mechanism on two neighboring datasets: $P \sim \mathcal{N}_{\beta}(\mu, \sigma)$ and $Q \sim \mathcal{N}_{\beta}(0, \sigma)$. 
\begin{proposition} \label{thm:prvs_for_gg}
    For $\sigma,\beta>0$, let  $Z \sim \mathcal{N}_{\beta}(0, \sigma)$, and let $\mu=\Delta f$. Then the PRVs for $GG_{\beta,\sigma}(f,D)$ are $X = \big(\frac{1}{\sigma}\big)^\beta \big(|Z|^\beta - |Z-\mu|^\beta \big)$ and $Y = \big(\frac{1}{\sigma}\big)^\beta \big(|Z-\mu|^\beta- |Z|^\beta \big)$. 
\end{proposition}
\begin{proof}
Given a noise distribution $Z$, and the associated differentially private mechanism, we derive the PRVs $X, Y$. This proof is of the same form as a similar derivation in \citet{microsoft_PRV_2}, specifically deriving the PRVs for the Gaussian mechanism.
The maximum difference in output of $f$ between two neighboring datasets is $\mu = \Delta f$, so we can abstract the worst-case pair of neighboring datasets with distributions  $P \sim \mathcal{N}_{\beta}(\mu, \sigma)$ and $Q \sim \mathcal{N}_{\beta}(0, \sigma)$.

Then for $t\sim Q =\mathcal{N}_\beta(0,\sigma)$, we can define the PRV $Y$ as the following distribution: 
    \begin{align*}
        Y & \sim \log\big(\frac{Q(t)}{P(t)} \big) = \log  \big(\frac{\exp(-t^\beta)}{-|t-\mu|^\beta} \big) = |t-\mu|^\beta - |t|^\beta. 
    \end{align*}

    By symmetry for $X$, for $t \sim P =\mathcal{N}_\beta(\mu ,\sigma)$, we can define the PRV $X$ as the distribution:
    \begin{align*}
        X \sim \log\big(\frac{Q(t)}{P(t)} \big) = \log  \big(\frac{-|t-\mu|^\beta}{\exp(-t^\beta)} \big) = |t|^\beta - |t-\mu|^\beta.
    \end{align*}
\end{proof}

    We note that while it is not necessarily true that $X = -Y$ for all private mechanisms, it is true for the Generalized Gaussian mechanism, as the proof above holds for all values of $\beta$ or $\sigma$.

\subsection{Dimension Independence in Multi-Dimensional PRVs}\label{appendix:dim_independence_GG}

We first derive the PRVs for the multi-dimensional Generalized Gaussian Mechanism, and then show that they are dimension independent for $1 \leq \beta \leq 2$, using the $\ell_{\beta}$ sensitivity.\footnote{We thank Matthew Joseph for pointing out the error in an earlier version of our paper, which incorrectly claimed dimension independence for all $\beta$. This property only holds for $1 \leq \beta \leq 2$.}

 As shown in \Cref{thm:prvs_for_gg}, the PRVs for the single dimensional GG Mechanism are $Y = |Z - \mu|^\beta - |Z|^\beta$ and $X = |Z|^\beta- |Z - \mu|^\beta$ where $Z \sim \mathcal{N}_{\beta}(\mu =\Delta f, \sigma=1)$. 

As shown in Section \ref{sec:prv_accountant_background}, the definition of a PRV is random variable generated from the probability of sampling a particular sample $t$, 
\begin{align*}
        Y \sim \log\big(\frac{Q(t)}{P(t)} \big) \textrm{where } t\sim Q \quad \text{ and } \quad  X \sim \log\big(\frac{Q(t)}{P(t)} \big) \textrm{where } t\sim P.
\end{align*}
This definition of a PRV does not change if $t$ is a scalar or a vector --- denoted $\vec{t} \in \mathbb{R}^d$, where $t_i$ is a scalar from $i$-th dimension of $\vec{t}$. We now observe that the multi-dimensional probability distribution is equal to the product distribution of the single-dimensional PDFs, since the multidimensional GG mechanism is sampled independently for each dimension: 
\begin{align*}
    \Pr(\vec{t}) = \prod^d_{i=1} \Pr(\vec{t}_i) \propto \prod^d_{i=1} \exp(- |\vec{t}_i|^\beta / \sigma ).
\end{align*}

Using this fact, we can then compute the PRVs of the multi-dimensional GG Mechanism.
Let $\vec{\mu}$ be a $d$-dimensional vector such that $||\vec{\mu}||_\beta = \Delta f$, and let $\mu_i$ be the $i$-th dimension of $\vec{\mu}$. Since each dimension of the noise is sampled independently in the GG mechanism, we denote the PRVs for a $d$ dimensional Generalized Gaussian mechanism as $Y_d$ and $X_d$. For $\vec{t} \sim Q$, these can be written as:
\[
       Y_d \sim  \log \prod_{i=1}^d \big( \frac{\exp(-|t_i|^\beta/\sigma)}{\exp(-|t_i - \mu_i|^\beta/\sigma)} \big) = \sum_{i=1}^d (|t_i - \mu_i|^\beta -|t_i|^\beta  ) / \sigma, \]
and by symmetry,    
\[
X_d \sim \sum_{i=1}^d (|t_i|^\beta -|t_i - \mu_i|^\beta  ) / \sigma.
\]

To connect this to $Y_d$, let $D,D'$ be neighboring datasets and
let $\vec{\mu}=f(D)-f(D')$ with $\|\vec{\mu}\|_\beta \le \Delta f$.
Let $Q$ denote the mechanism's output distribution on $D$ and $P$ the output distribution on $D'$.
For an additive-noise mechanism, these distributions are shifts of one another, so
\[
P(\vec t)=Q(\vec t-\vec{\mu}).
\]
Hence the privacy loss random variable (PRV) is exactly the log-likelihood ratio appearing above:
when $\vec t\sim Q$,
\[
\log\frac{Q(\vec t)}{P(\vec t)}
=\log\frac{Q(\vec t)}{Q(\vec t-\vec{\mu})}
=\frac{1}{\sigma}\Big(\|\vec t-\vec{\mu}\|_\beta^\beta-\|\vec t\|_\beta^\beta\Big),
\]
which is precisely $Y_d$; when $\vec t\sim P$ the same log-ratio gives $X_d$ (the reverse direction).

For a one-dimensional mechanism, $\mu$ in the PRV calculation must equal the sensitivity $\Delta f$ of the function being privately evaluated. However, for multi-dimensional mechanisms, $\vec{\mu}$ must instead be a vector with norm $||\vec{\mu}||_{\beta} = \Delta f$ that maximizes the privacy loss random variables $X$ and $Y$. As we will see, there are many possible $\vec{\mu}$ vectors which meet this requirement.

Below, we show that if sensitivity is measured with respect to the $\ell_{\beta}$ norm, then the PRVs for a multi-dimensional Generalized Gaussian mechanism are simply the PRVs of the corresponding the single-dimensional GG mechanism; this motivates our choice of $||\cdot||_{\beta}$ sensitivity, rather than fixing the sensitivity for all mechanisms.

Recall the definition of the $\ell_\beta$ norm: $|| \vec{x}||_{\beta} = (\sum_{i} |x_i|^\beta )^{1/\beta}$
To start, we observe that $Y_d= \sum_{i=1}^d (|t_i - \mu_i|^\beta -|t_i|^\beta  ) / \sigma$ can be rewritten as 
\[
 Y_d = \frac{1}{\sigma} \left[ || \vec{t}-\vec{\mu} ||_{\beta}^\beta - || \vec{t} ||_{\beta}^\beta \right].
\] 

Since differential privacy requires a worst-case bound over all pairs of neighboring databases, it requires bounding $\max\{\frac{Q(\vec{t})}{Q(\vec{t} - \vec{\mu})},\frac{Q(\vec{t}- \vec{\mu})}{Q(\vec{t})}\}$, for all difference vectors $\mu$ satisfying the sensitivity bound $||\vec{\mu}||_\beta = \Delta f$.  
Then the set of $\vec{\mu}$ over which this difference must be maximized is all points in an $\ell_{\beta}$-ball of radius $\Delta f$.

More formally, let $\mathcal{C}$ be this set of points satisfying satisfying $||\vec{\mu} ||_{\beta} \leq \Delta f$. Then the value of $\vec{\mu}$ that maximizes this difference is:
\[
    \arg\max_{\vec{\mu} \in \mathcal{C}} \left\{Y_d = \frac{1}{\sigma}(||\vec{t}-\vec{\mu}||_\beta^\beta - ||\vec{t}||_\beta^\beta ) \right\}.
\]

Since $||\vec{t}||_\beta^\beta $ does not depend on $\vec{\mu}$, then the expression is also maximized by:

\[
    \arg\max_{\vec{\mu} \in \mathcal{C}} \left\{ \frac{1}{\sigma}(||\vec{t}-\vec{\mu}||_\beta^\beta ) \right\}.
\]

For $1\leq \beta <2$, we observe that any $\vec{\mu}$ on the $\ell_{\beta}$ ball of radius $\Delta f$, i.e., $\vec{\mu} \in \{\vec{x} : ||\vec{x}||_\beta = \Delta f \}$, will satisfy this maximal-difference constraint because all such points are exactly $\Delta f$ away from the origin in $\ell_{\beta}$-distance. Thus $Y_d = \frac{1}{\sigma}\left[\| \vec{t}-\vec{\mu} \|_\beta^\beta -  \| \vec{t} \|_\beta^\beta \right] $ will be identical for any $\vec{\mu}$ on the $\ell_\beta$-ball. This includes, for example, the one-hot vector $\vec{\mu} = \langle1,0, \ldots, 0\rangle$.

Plugging in this choice of $\vec{\mu}$, we see that the PRVs for the multidimensional GG mechanism are the same as the single-dimensional GG mechanism: 
\begin{align*}
       Y_d &\sim \sum_{i=1}^d (|t_i - \mu_i|^\beta -|t_i|^\beta  ) / \sigma \\
       & = \bigg( |t_1 - \Delta f|^\beta - |t_1|^\beta  + 0 + \ldots + 0 \bigg)  / \sigma \\
       & = (|(t_1 - \Delta f)|^\beta - |t_1|^\beta ) / \sigma \\
       & = Y.
\end{align*}

The same holds for $X_d$ and $X$:
\begin{align*}
       X_d &\sim \sum_{i=1}^d (|t_i|^\beta - |t_i - \mu_i|^\beta  ) / \sigma \\
       & = \bigg( | t_1|^\beta - |t_1 -\Delta f|^\beta  + 0 + \ldots + 0 \bigg)  / \sigma \\
       & = (| t_1|^\beta- |t_1 - \Delta f|^\beta ) / \sigma \\
       & = X.
\end{align*}

Thus for $\beta$-GG mechanism with sensitivity measured with a $\ell_\beta$-norm for $\beta \in [1,2]$, the PRV for the multi-dimensional GG distribution is equivalent to the PRV for a single-dimensional GG distribution, when the sensitivity is also measured with the $\ell_{\beta}$-norm.

To see why this argument fails for $\beta > 2$, recall that $\vec{\mu}$ must be a 
vector that maximizes the privacy loss random variables $X$ and $Y$. For $\beta > 2$, only the diagonal vectors along the hypersphere (that is, points of the form 
$\{\pm 1, \pm 1, \ldots, \pm 1\} \cdot \frac{1}{\sqrt{d}}$) satisfy this 
maximal-difference constraint on the $\ell_{\beta}$ unit sphere.  This means that for $\beta > 2$, the PRV cannot be reduced to a scalar distribution and instead depends on  dimension. Thus, the only known method to sample points from this family is to sample a $d$-dimensional vector. 

In the next subsection (\Cref{appendix:privacy_in_practice}), we introduce the Sampled PRV accountant, which estimates privacy guarantees by constructing an empirical CDF from samples of the PRV distribution; obtaining tight error bounds requires drawing very large numbers of such samples. When the PRV is dimension-dependent, those samples must be $d$-dimensional vectors, which is computationally infeasible for ML models with $d > 10^6$ parameters. Dimension-independence allows $1$-dimensional samples to be drawn instead, which results in substantial computational savings.

\subsection{Privacy of the Sampled PRV Accountant} \label{appendix:privacy_in_practice}

In this section we formally introduce the Sampled PRV accountant, and provide theoretical guarantees for its privacy and accuracy as a PRV accountant (Algorithm~\ref{alg:sampledPRV}). The algorithm works along the same lines as the PRV accountant from~\cite{microsoft_PRV_2}, but instead of composing the actual CDFs of the PRVs, this algorithm composes the empirical distribution functions (EDFs) from samples from those PRV distributions. This is an (high-probability) approximation to the guarantees of the PRV accountant from~\cite{microsoft_PRV_2} with the uncertainty coming from the sampling error of the points obtained from the PRV distributions. Our Sampled PRV accountant (SampledComposePRV, Algorithm~\ref{alg:sampledPRV}) extends the ComposePRV algorithm to work even when the CDF of the PRVs is not known, but the algorithm instead only has sample-access to the distribution.

Our approach is similar to that of~\cite{microsoft_PRV_2}, where we first discretize the real space and map the samples on to this grid (with our choice of width $h$) using the same mapping as~\cite{microsoft_PRV_2}, while truncating the samples to within a fixed interval $[-L,L]$. Note that the truncation from~\cite{microsoft_PRV_2} removes the probability mass from outside $[-L,L]$ and then normalizes the truncated distribution. For our case, we may either construct the EDF only using the samples from the PRV that lie within $[-L,L]$ or we may construct the EDF from all the samples and truncate in the same way as in~\cite{microsoft_PRV_2} -- both ways would be equivalent, and the distribution of the EDF would be the same as in the case where the EDF is constructed from the original discretized and truncated PRV. After constructing the EDF using the discretized and truncated samples for each PRV, we convolve all the EDFs using FFT with modulo arithmetic in the same way as~\cite{microsoft_PRV_2}, i.e., using the operator $A \oplus B \equiv A+B~(\text{mod } 2L) $.

First we need an algorithm $\texttt{sample}_n(Y,n, L,h)$ (Algorithm~\ref{alg:sampleFunction}) which takes in a random variable $Y$, an integer $n$, real-valued domain size $L$, and real-valued bin-width $h$, and returns a random variable $Y_{n,L,h}$ as output. This algorithm generates $n$ samples from $Y$, then constructs a discretized random variable over the domain $[-L,L]$ with distribution equal to the empirical CDF of the samples. 

\begin{algorithm}[ht]
    \caption{Sample Function $\texttt{sample}(Y,n,L,h)$} \label{alg:sampleFunction}
    \textbf{Input:} Integer $n$, real-valued bin-width $h$, real-valued domain-range $L$, random variable $Y$ \\
    \textbf{Output:} Random variable $Y_{n,L,h}$ constructed from the empirical CDF of $n$ samples from $Y$
    \begin{algorithmic}[1]
        \State Set $m = \tfrac{L-\frac{h}{2}}{h}$
        \State Generate independent samples from $Y$ until we have $2n$ samples that lie within $[-L,L]$: $\{y_1, y_2, \dots, y_{2n}\} \sim_{i.i.d.} Y$ with rejection sampling
        \State Define $Y^1 = \{y_1,\dots,y_n\}$ and $Y^2 = \{y_{n+1},\dots,y_{2n}\}$
        \For{$i = -m$ to $m$}
            \State Set $p_i = \text{CDF}_{Y^1}(ih + h/2) - \text{CDF}_{Y^1}(ih - h/2)$
        \EndFor
        \State Set $\vec{q} = \tfrac{1}{\sum_{i=1}^{2n}{p_i}}\langle p_{1}, \dots, p_{2n}\rangle$ 
        \State Set $\tilde{\mu} = \tfrac{1}{n}\sum_{j=n}^{2n}{y_j} - \sum_{i=-m}^{m}{ihq_i}$
        \State Set $\hat{\mu} = \max\{\min\{\tilde{\mu},h/2\},0\}$
        \State Set $\tilde{Y} = \Big\{ih + \hat{\mu} \text{ w.p. } q_i \text{ for } -m \leq i \leq m$
        \State \textbf{Return} $Y_{n,L,h}$
    \end{algorithmic}
\end{algorithm}

Finally, we can present our SampledComposePRV algorithm (Algorithm~\ref{alg:sampledPRV}). This algorithm takes in sample-access mechanisms to $k$ PRVs $Y_1,\ldots,Y_k$, where each $Y_i$ is the PRV for a single mechanism. For each $Y_i$, the algorithm generates $n$ samples (that lie within $[-L,L]$, using Algorithm~\ref{alg:sampleFunction}) and constructs an empirical (PRV) distribution $Y_{i,n}$ from the samples. Let $\tilde{Y}_i$ be the truncated and discretized version of $Y_i$ as described in~\cite{microsoft_PRV_2}. Suppose $Y= \sum_{i=1}^k Y_i$ is the true composed PRV, $\tilde{Y} = \sum_{i=1}^{k}{\tilde{Y_i}}$ is the composition of the discretized and truncated PRVs, and $Y_{n}= \sum_{i=1}^k Y_{i,n}$ is the composed empirical PRV. The last summation $Y^n$ is computed correctly with high probability by our algorithm via FFT as $\hat{Y}_n = \bigoplus_{i=1}^{k}{Y_{i,n}}$, where the $\oplus$ operator is defined via modular arithmetic as above.
    
\begin{algorithm}[htb]
    \caption{SampledComposePRV($\{Y_i\}_{i\in[k]}, h,L,n$)} \label{alg:sampledPRV}
    \textbf{Input:} Sample-access to PRVs $Y_1, Y_2, \ldots Y_k$ , mesh size $h$, truncation parameter $L$ that is a multiple of $h$, number of samples $n$ \\
    \textbf{output} PDF of an approximation $\tilde{Y}$ of $Y= \sum_{i=1}^k Y_i$,  supported on a grid over $[-L,L]$ with bin-width $h$    
    \begin{algorithmic}[1]
        \For{$i=1$ to $k$}
            \State $Y_{i,n} \leftarrow \texttt{sample}(Y_i,n,L,h)$ 
        \EndFor   
        \State Compute PDF of $\hat{Y}_{n} = \bigoplus\limits_{i=1}^{k}{Y_{i,n}}$, where $\oplus$ is the convolution from~\cite{microsoft_PRV_2}.
        \State Compute $\delta_{\hat{Y}_{n}}(\epsilon) = \mathbb{E}_{\hat{Y}_{n}} \big[  \big( 1- e^{\epsilon - \hat{Y}_{n}}\big)_{+}\big]$ for all $\epsilon \in [0,L]$.
        \State Return $\hat{Y}_{n}, \delta_{\hat{Y}_{n}}$.
    \end{algorithmic}
\end{algorithm}

We present the guarantees of our privacy accountant in the following theorem, which is the main result of this section.
\begin{theorem}\label{thm:sampled_prv_error}
    Let $s,t>0$ be some fixed parameters. Let $\mathcal{M}_1,\mathcal{M}_2,\dots,\mathcal{M}_k$ be DP algorithms with privacy curves $\delta_{\mathcal{M}_i}(\epsilon)$. Let $Y_i$ denote the PRV corresponding to $\mathcal{M}_i$, such that $\delta_{\mathcal{M}_i}(\epsilon)=\delta_{Y_i}(\epsilon)$ for $\epsilon \geq 0$. Let $\mathcal{M}$ be the (adaptive) composition of $\mathcal{M}_1,\mathcal{M}_2,\dots,\mathcal{M}_k$, and let $\delta_{\mathcal{M}}(\epsilon)$ be its privacy curve. Let $\hat{Y}_n$ (where $n>0$) be the approximation of $Y=\sum_{i=1}^k Y_i$ produced by SampledComposePRV with mesh size $h>0$ and truncation parameter $L>0$. Let $2n$ denote the number of samples used in \texttt{sample} (Algorithm \ref{alg:sampleFunction}). Then,
    \[
        \delta_{\hat{Y}_n}\left(\epsilon+\tau\right) - \eta \leq \delta_{\mathcal{M}}(\epsilon) \le \delta_{\hat{Y}_n}\left(\epsilon-\tau\right)+\eta,
    \]
    where
    \begin{align*}
        \eta &= 2\sum_{i=1}^{k}{\Pr[|Y_i| \geq L]} + 4e^{-\frac{2s^2}{kh^2}} + 4ke^{-\frac{nt^2}{2L^2}} + 8ke^{-\frac{nt^2}{2}} + \Pr\left[\left|\textstyle \sum\limits_{i=1}^{k}{Y_i}\right| \geq L-t\right] + 2k\left(t + \sqrt{\tfrac{L}{nh}}\right),
    \end{align*}
    and
    \begin{align*}
        \tau = s + k\left(t + 2L\left(\tfrac{t}{2} + \sqrt{\tfrac{L}{nh}}\right)\right) + 2k\left(\tfrac{t}{2} + \sqrt{\tfrac{L}{nh}}\right).
    \end{align*}

    Additionally, SampleComposePRV takes $O\left( b \frac{L}{h} \log(\frac{L}{h}) + bn\right)$ time, where $b$ is the number of distinct algorithms among $\mathcal{M}_1,\mathcal{M}_2,\dots,\mathcal{M}_k$.    
\end{theorem}

In order to prove Theorem~\ref{thm:sampled_prv_error}, we will analyze a series of couplings, some of which have already been defined in~\cite{microsoft_PRV_2}. For that, we first define \emph{coupling approximations} for measuring closeness between two distributions, which were used by~\cite{microsoft_PRV_2}.

\begin{definition}[Coupling Approximation \citep{microsoft_PRV_2}]
    Given two random variables $Y_1, Y_2$, we write that $|Y_1 - Y_2| \leq_{\eta} c$ if there exists a coupling between $Y_1, Y_2$ such that $\Pr[|Y_1 - Y_2| >c] \leq \eta$.
\end{definition}

As a reminder to the reader, $Y$ denotes the sum of all the original PRVs ($Y_i$), $\tilde{Y}$ denotes the sum of all the truncated and discretized PRVs ($\tilde{Y}_i$), $Y^n$ denotes the sum of all the empirical PRVs, each created by $2n$ discretized and truncated samples ($Y_{i,n}$), and $\hat{Y}_n$ denotes the convolution of all these sampled PRVs ($Y_{i,n}$). First we will show that $|Y^n - \hat{Y}_n| \leq_{\eta_2} \tau_2$ for some $\eta_2,\tau_2 \geq 0$. Next, we will prove that $|\tilde{Y} - Y^n| \leq_{\eta_1} \tau_1$ for some $\eta_1,\tau_1 \geq 0$. We finally prove that $|Y - \tilde{Y}| \leq_{\eta_0} \tau_0$ for some $\eta_0,\tau_0 \geq 0$, for our specific algorithmic setting that is slightly different from the one in~\cite{microsoft_PRV_2}. We will then invoke the following result from~\cite{microsoft_PRV_2} that defines a triangle inequality on coupling approximations.

\begin{lemma}[Lemma~C.1, \citep{microsoft_PRV_2}]
    \label{lem:triangle_inequality_couplings}
    Suppose $X, Y, Z$ are random variables such that $|X - Y| \leq_{\eta_1} c_1$ and $|Y - Z| \leq_{\eta_2} c_2$. Then, 
    $$|X - Z| \leq_{\eta_1 + \eta_2} c_1 + c_2.$$
\end{lemma}

This would give us $|Y-\hat{Y}_n| \leq_{\eta} \tau$, where $\eta=\eta_0+\eta_1+\eta_2$ and $\tau=\tau_0+\tau_1+\tau_2$. After that, we will use the following result from~\cite{microsoft_PRV_2} that translates guarantees from coupling approximations to privacy curves.

\begin{lemma}[Lemma~5.2, \citep{microsoft_PRV_2}] \label{lem:coupling_to_privacy_curves}
	If $Y$ and $\tilde{Y}$ are two PRVs such that $|Y-\tilde{Y}|\le_\eta c$, then for every $\eps>0$, 
	$$\delta_{\tilde{Y}}(\eps+c)-\eta \le \delta_{Y}(\eps) \le \delta_{\tilde{Y}}(\eps-c)+\eta.$$
\end{lemma}

We start by proving that $|Y^n - \hat{Y}_n| \leq_{\eta_2} \tau_2$ for some $\eta_2,\tau_2 \geq 0$ and some given $\alpha > 0$.
\begin{lemma}\label{lem:coupling-sample-to-conv}
    Let $\tilde{Y}_1,\dots, \tilde{Y}_k, Y_{1,n},\dots,Y_{k,n}, Y^n, \hat{Y}_n$ be defined as above. Suppose $d_{\text{TV}}(\tilde{Y}_i,Y_{i,n}) \leq \alpha$ for each $1 \leq i \leq k$. Then,
    \[
        \left|Y^n - \hat{Y}_n\right| \leq_{\eta_2} 0,
    \]
    where $\eta_2 = \Pr\left[|\sum_{i=1}^{k}{\tilde{Y}_i}| \geq L\right] + 2k\alpha$.
\end{lemma}
\begin{proof}
    Let $\lambda = \Pr\left[|\sum_{i=1}^{k}{\tilde{Y}_i}| \geq L\right]$ and $\lambda_j = \Pr\left[|\sum_{i=1}^{j}{\hat{Y}_{i,n}} + \sum_{i=j+1}^{k}{\tilde{Y}_i}| \geq L\right]$, for $1 \leq j \leq k$. We prove by induction that $\lambda_j \leq \lambda + 2j\alpha$. Let $S$ be the set of all the points in the discretized version of $[-L,L]$.

    For the base case ($j=1$), consider the joint distribution $(Y_{1,n}, \tilde{Y}_2, \dots, \tilde{Y}_k)$. Then we have the following:
    \begin{align*}
        \left|\lambda_1 - \lambda\right| &= \left|\sum\limits_{s \in S}\Pr\left[Y_{1,n} = s\right]\cdot\Pr\left[\left|\sum\limits_{i=2}^{k}{\tilde{Y}_i} + s\right| \geq L\right] - \sum\limits_{s \in S}\Pr\left[\tilde{Y}_1 = s\right]\cdot\Pr\left[\left|\sum\limits_{i=2}^{k}{\tilde{Y}_i} + s\right| \geq L\right]\right|\\
        &= \left|\sum\limits_{s \in S}\left(\Pr\left[Y_{1,n} = s\right] - \Pr\left[\tilde{Y}_1 = s\right]\right)\cdot\Pr\left[\left|\sum\limits_{i=2}^{k}{\tilde{Y}_i} + s\right| \geq L\right]\right|\\
        &\leq \left|\sum\limits_{s \in S}\left(\Pr\left[Y_{1,n} = s\right] - \Pr\left[\tilde{Y}_1 = s\right]\right)\right| \times 1 \tag{H\"{o}lder's inequality.}\\
        &\leq 2\alpha
    \end{align*}
    This proves our base case. Next, we assume that $|\lambda_j-\lambda| \leq 2j\alpha$ for some $j = j' \geq 1$. We will show that $|\lambda_j-\lambda| \leq 2j\alpha$ for $j=j'+1$. The argument is very similar to that for the base case.
    \begin{align*}
        \left|\lambda_j - \lambda_{j-1}\right| &= \Bigg|\sum\limits_{s \in S}\Pr\left[Y_{j,n} = s\right]\cdot\Pr\left[\left| \sum\limits_{i=1}^{j-1}{Y_{i,n}}+ \sum\limits_{i=j+1}^{k}{\tilde{Y}_i} + s\right| \geq L\right]\\
            &~~~~~~~~~~~~- \sum\limits_{s \in S}\Pr\left[\tilde{Y}_j = s\right]\cdot\Pr\left[\left|\sum\limits_{i=1}^{j-1}{Y_{i,n}}+ \sum\limits_{i=j+1}^{k}{\tilde{Y}_i} + s\right| \geq L\right]\Bigg|\\
        &= \left|\sum\limits_{s \in S}\left(\Pr\left[Y_{j,n} = s\right] - \Pr\left[\tilde{Y}_j = s\right]\right)\cdot\Pr\left[\left| \sum\limits_{i=1}^{j-1}{Y_{i,n}}+ \sum\limits_{i=j+1}^{k}{\tilde{Y}_i} + s\right| \geq L\right]\right|\\
        &\leq \left|\sum\limits_{s \in S}\left(\Pr\left[Y_{j,n} = s\right] - \Pr\left[\tilde{Y}_j = s\right]\right)\right| \times 1 \tag{H\"{o}lder's inequality.}\\
        &\leq 2\alpha
    \end{align*}
    Therefore, by triangle inequality, we get the following.
    \[|\lambda_j - \lambda| = |\lambda_j - \lambda_{j-1} + \lambda_{j-1} - \lambda| \leq |\lambda_j - \lambda_{j-1}| + |\lambda_{j-1} - \lambda| \leq 2(j-1)\alpha + 2\alpha = 2j\alpha\]
    This completes the inductive step. Consequently, $|\lambda_k-\lambda| \leq 2k\alpha$, which implies that $\lambda_k \leq \lambda + 2k\alpha$.

    We finally use the following result from~\cite{microsoft_PRV_2} to complete the proof of Lemma \ref{lem:coupling-sample-to-conv}.
    \begin{lemma}[Lemma~C.7, \cite{microsoft_PRV_2}]
        Let $X_1, \dots, X_k$ be random variables supported on $[-L,L]$. Then,
        \[
            \left|\sum\limits_{i=1}^{k}{X_i} - \bigoplus\limits_{i=1}^{k}{X_i}\right| \leq_{\beta} 0,
        \]
        where $\beta = \Pr\left[|\sum_{i=1}^{k}{X_i}| \geq L\right]$.
    \end{lemma}
    This result would imply that $|Y^n - \hat{Y}_n| \leq_{\beta} 0$, where $\beta = \Pr\left[|\sum_{i=1}^{k}{Y_{i,n}}| \geq L\right]$. However, because of our claim above, $\beta \leq \Pr\left[|\sum_{i=1}^{k}{\tilde{Y}_i}| \geq L\right] + 2k\alpha$. Combining the two yields the required result.
\end{proof}

Next, we show that $|\tilde{Y} - Y^n| \leq_{\eta_1} \tau_1$ for some $\eta_1,\tau_1 \geq 0$.
\begin{lemma}\label{lem:coupling-dist-to-samples}
    Let $\tilde{Y}_1,\dots, \tilde{Y}_k, Y_{1,n},\dots,Y_{k,n}, \tilde{Y}, Y^n$ be defined as above. Then,
    \[
        \left|\tilde{Y} - Y^n\right| \leq_{\eta_1} \tau_1,
    \]
    where $\eta_1 = 2ke^{-nt^2/2}$, $\tau_1 = 2k\alpha$, and $\alpha = \tfrac{t}{2} + \sqrt{\tfrac{L}{nh}}$.
\end{lemma}
\begin{proof}
    Suppose it holds that for each $1 \leq i \leq k$, with probability at least $1-\beta/k$ (for some $0 < \beta < 1$), $d_{\text{TV}}(\tilde{Y}_i,Y_{i,n}) \leq \alpha$ for some $0 < \alpha < 1$. We will determine the values of $\alpha,\beta$ later in this proof. For $j \geq 1$, let $\tilde{\lambda}_j = \sum_{i=1}^{j}{\tilde{Y}_i}$ and $\lambda_j = \sum_{i=1}^{j}{Y_{i,n}}$. We will show by induction that $|\tilde{\lambda}_j - \lambda_j| \leq_{\frac{\beta j}{k}} 2j\alpha$. Let $S'$ be the set of all the points in the discretized version of $[-L,L]$ and $S$ be a similarly discretized version of $[-Lk,Lk]$.

    For the base case ($j=1$), $|\tilde{Y}_1-Y_{1,n}| \leq_{\frac{\beta}{k}} 2\alpha$ trivially holds due to our construction. Next, assume that the claim holds for some $j = j' \geq 1$. We show that $|\tilde{\lambda}_j - \lambda_j| \leq_{\frac{\beta j}{k}} 2j\alpha$ for $j = j'+1$.
    \begin{align*}
        \sum\limits_{s \in S}{\left|\Pr\left[\sum\limits_{i=1}^{j}{\tilde{Y}_i} = s\right] - \Pr\left[\sum\limits_{i=1}^{j}{Y_{i,n}} = s\right]\right|} &= \sum\limits_{s \in S}\Bigg|\Pr\left[\sum\limits_{i=1}^{j}{\tilde{Y}_i} = s\right] - \Pr\left[\sum\limits_{i=1}^{j-1}{\tilde{Y}_i} + Y_{j,n} = s\right]\\
            &~~~~~~~~~+ \Pr\left[\sum\limits_{i=1}^{j-1}{\tilde{Y}_i} + Y_{j,n} = s\right] - \Pr\left[\sum\limits_{i=1}^{j}{Y_{i,n}} = s\right]\Bigg|\\
        &\leq \sum\limits_{s \in S}{\left|\Pr\left[\sum\limits_{i=1}^{j}{\tilde{Y}_i} = s\right] - \Pr\left[\sum\limits_{i=1}^{j-1}{\tilde{Y}_i} + Y_{j,n} = s\right]\right|}\\
            &~~~~~~~~~+ \sum\limits_{s \in S}{\left|\Pr\left[\sum\limits_{i=1}^{j-1}{\tilde{Y}_i} + Y_{j,n} = s\right] - \Pr\left[\sum\limits_{i=1}^{j}{Y_{i,n}} = s\right]\right|}\\
        &\leq 2\alpha + 2(j-1)\alpha\\
        &= 2j\alpha
    \end{align*}
    The first inequality follows from the triangle inequality. We now reason about the last inequality. The first sum of absolute values holds because the two terms in the absolute value differ only at $j$, i.e., one is $\tilde{Y}_j$ and the other is $Y_{j,n}$. The TV-distance between them is at most $\alpha$ with probability at least $1-\beta/k$, therefore, the summation is at most $2\alpha$. The second summation is at most $2(j-1)\alpha$ because the two terms in the absolute value differ on all $1 \leq i \leq j-1$, and by the inductive hypothesis, the summation is at most $2(j-1)\alpha$. Applying the union bound, this proves the claim. This also implies that 
    \begin{equation}\label{eq.alphabeta}
      |\tilde{Y}-Y_{n}| \leq_{\beta} 2k\alpha.      
    \end{equation}

    Next, we invoke the following lemma from~\cite{canonne2020shortnotelearningdiscrete}, which bounds the TV-distance between a distribution over a discrete domain and the empirical distribution formed by i.i.d. samples from the distribution, to determine $\alpha,\beta$.
    \begin{lemma}[Proof of Theorem~1, \citep{canonne2020shortnotelearningdiscrete}]\label{lem:sampling_coupling}
        Given an unknown distribution \( X \in \Delta([b]) \), let \( \hat{X} \) be the empirical distribution obtained from \( n \) independent samples drawn from $X$. Then the expected total variation distance $d_{\text{TV}} $ between \( \hat{X} \) and \( X \), satisfies:
        \[
            \mathbb{E}[ d_{\text{TV}} (X, \hat{X}) ] \leq \frac{1}{2\sqrt{n}} \sum_{i=1}^{b} \sqrt{ X(i) } \leq \frac{1}{2} \sqrt{\frac{b}{n}}.
        \]
        Furthermore, this expected result can be converted into a high probability guarantee:
        \[
            \Pr \left[ | d_{\text{TV}} (X, \hat{X}) - \mathbb{E}[ d_{\text{TV}} (X,\hat{X}) ] | \geq \frac{t}{2} \right] \leq 2 e^{ - \frac{1}{2} n t^2 } \quad \forall t>0.
        \]
    \end{lemma}
    We instantiate Lemma~\ref{lem:sampling_coupling} with $2e^{-nt^2/2} = \beta/k$, or equivalently, $\beta = 2ke^{-nt^2/2}$. Note that there are $2L/h + 1 < 4L/h$ bins in our discrete distribution, so we can also set $b = 4L/h$. Therefore, this lemma implies that for each $1 \leq i \leq k$, with probability at least $1-\beta/k$, 
    \[\alpha \coloneqq d_{\text{TV}}(\tilde{Y}_i,Y_{i,n}) \leq \frac{t}{2} + \sqrt{\frac{L}{nh}}.\]

    Substituting these values of $\alpha$ and $\beta$ in Equation \eqref{eq.alphabeta} gives the statement of Lemma \ref{lem:coupling-dist-to-samples}.
\end{proof}

We next show that $|Y - \tilde{Y}| \leq_{\eta_0} \tau_0$ for some $\eta_0,\tau_0 \geq 0$.
\begin{lemma}\label{lem:coupling-dist-to-disc}
    Let $Y_1,\dots,Y_k, \tilde{Y}_1,\dots,\tilde{Y}_k, Y, \tilde{Y}$ be defined as above. Then,
    \[
        \left|\tilde{Y} - Y\right| \leq_{\eta_0} \tau_0,
    \]
    where $\eta_0 = \sum_{i=1}^{k}{\Pr[|Y_i| \geq L]} + 2e^{-\frac{2s^2}{kh^2}} + 2k(e^{-\frac{nt^2}{2L^2}} + e^{-\frac{nt^2}{2}})$, $\tau_0 = s + k\left(t + 2L\left(\tfrac{t}{2} + \sqrt{\tfrac{L}{nh}}\right)\right)$.
\end{lemma}
\begin{proof}
    This lemma is an analogue of Corollary~C.8 from~\cite{microsoft_PRV_2}. To prove this lemma, we will rely on Hoeffding's inequality, stated as follows.
    \begin{lemma}[Hoeffding's Inequality]\label{lem:hoeffding}
        Let $X_1,\dots,X_n$ be i.i.d.~random variables and $L > 0$, such that for each $1 \leq i \leq n$, $X_i \in [-L,L]$. Let $S = \tfrac{1}{n}\sum_{i=1}^{n}[X_i]$. Then for all $t >0$,
        \[
            \Pr\left[|S - \mathbb{E}[S]| \geq t\right] \leq 2e^{-\frac{nt^2}{2L^2}}.
        \]
    \end{lemma}

    For each $1 \leq i \leq k$, let $Y^L_i$ be the truncated (but not discretized) and conditional version of $Y_i$ obtained via rejection sampling, and let $X^L_{i,1},\dots,X^L_{i,n}$ be the truncated (but not discretized) samples from $Y^{L}_i$ used to determine $\hat{\mu}$ in Algorithm~\ref{alg:sampleFunction} (which in Algorithm~\ref{alg:sampleFunction}, refers to the samples $y_{n+1},\dots,y_{2n}$ collected for any given PRV, $Y_i$), and let $\hat{X}_{i,j}$ be the version of $X^L_{i,j}$ mapped on to the midpoint of the interval it lies in. Using Lemma~\ref{lem:hoeffding}, with probability at least $1-2ke^{-\frac{nt^2}{2L^2}}$, for each $1 \leq i \leq k$,
    $$\mathbb{E}[X^L_i] - t \leq X^L_i \coloneqq \frac{1}{n}\sum\limits_{j=1}^{n}{X^L_{i,j}} \leq \mathbb{E}[X^L_i] + t.$$

    Now, for a fixed $1 \leq i \leq k$, we will refer to the underlying discrete distribution of $\hat{X}_{i,j}$ as $P_i$. This implies by linearity of expectations that $\mathbb{E}[X^L_i] = \mathbb{E}[P_i]$. Let $Q_i$ be the (empirical) discrete distribution defined by $\hat{X}_{i,1},\dots,\hat{X}_{i,n}$. We are interested in bounding $|\mathbb{E}[P_i] - \mathbb{E}[Q_i]|$. 
    
    First, we show that if $d_{\text{TV}}(P_i,Q_i) \leq \alpha_0$, then $|\mathbb{E}[P_i] - \mathbb{E}[Q_i]| \leq 2L\alpha_0$. Let $S$ be the set of points in $[-L,L]$ over which $P_i$ has been defined.
    \begin{align*}
        |\mathbb{E}[P_i] - \mathbb{E}[Q_i]| &= \left|\sum\limits_{s \in S}{s(p_i(s) - q_i(s))}\right|\\
        &\leq L\cdot\sum\limits_{s \in S}{\left|p_i(s) - q_i(s)\right|} \tag{H\"{o}lder's inequality.}\\
        &\leq 2L\alpha_0
    \end{align*}
    We now use Lemma~\ref{lem:sampling_coupling} to bound $\alpha_0$. As in Lemma~\ref{lem:coupling-dist-to-samples}, using the union bound, we have that that with probability at least $1-2ke^{-nt^2/2}$, for each $1 \leq i \leq k$, $|\mathbb{E}[P_i] - \mathbb{E}[Q_i]| \leq 2L\alpha_0$, where
    \[\alpha_0 \leq \frac{t}{2} + \sqrt{\frac{L}{nh}}.\]

    In \cite{microsoft_PRV_2}, Algorithm~2 defines a quantity, $\mu \in [0,h/2]$, which depends on the parameters of the actual distributions of the PRVs and acts as an additive offset to define the truncated and discretized distribution from the truncated (but not discretized) distribution ($Y^L_i$): $\mu = \mathbb{E}[X^L_i] - \mathbb{E}[P_i]$. This offset is determined using the CDFs of the actual PRVs, and is used to satisfy certain technical requirements for their proof. Since our proofs have similar technical requirements, we approximate $\mu$ empirically using random samples ($y_1,\dots,y_n$ in Algorithm~\ref{alg:sampleFunction} for a given PRV $Y_i$) from the PRVs (which we call, $\hat{\mu}$) in Algorithm~\ref{alg:sampleFunction}, and show using the above that $\mu$ and $\hat{\mu}$ are close to one another.
    From our construction, $\hat{\mu} = \max\{\min\{X^L_i - \mathbb{E}[Q_i],h/2\},0\}$. The above implies using triangle inequality that with probability at least $1-2k(e^{-\frac{nt^2}{2L^2}} + e^{-\frac{nt^2}{2}})$,
    $$\mu - (t + 2L\alpha_0) \leq \hat{\mu} \leq \mu + (t + 2L\alpha_0).$$ Now, let $\gamma = t + 2L\alpha_0$; this implies that $|\hat{\mu}-\mu| \leq \gamma$ with high probability. 
 
    We first prove two facts, which are analogs of Lemmas~C.5 and~C.6 from~\cite{microsoft_PRV_2}. The first fact (analog of Lemma~C.5 in \cite{microsoft_PRV_2}) is that with high probability, for each $1 \leq i \leq k$, (1) $\mathbb{E}[\tilde{Y}_i] = \mathbb{E}[Y^L_i] + \tau_i$ for some $|\tau_i| \leq \gamma$, (2) $|Y^L_i - (\tilde{Y}_i - \hat{\mu})| \leq_{0} \tfrac{h}{2}$, and (3) $|Y^L_i - Y_i| \leq_{\lambda_i} 0$ (where $\lambda_i = \Pr[|Y_i| \geq L]$).
    Part (3) has already been shown in~\cite{microsoft_PRV_2}, so we will focus on the first two.
    
    We start by proving part (1). We have shown above that $|X^L_i - \mathbb{E}[Q_i]| \leq \gamma$ with high probability, so it must be the case that $\hat{\mu} = \mu + \tau_i$ (for some $|\tau_i| \leq \gamma$). Our construction (using the fact that $\mathbb{E}[Q_i] = \sum_{j=-m}^{m}{jhq_j}$) implies the following:
    \[\mathbb{E}[\tilde{Y}_i] = \hat{\mu} + \sum\limits_{j=-m}^{m}{jhq_j} = \mu + \tau_i + \sum\limits_{j=-m}^{m}{jhq_j} = \tau_i + \left(\mathbb{E}[Y^L_i] - \sum\limits_{j=-m}^{m}{jhq_j}\right) + \sum\limits_{j=-m}^{m}{jhq_j} = \tau_i + \mathbb{E}[Y^L_i].\]
    
    To prove part (2), we argue in a similar way as in~\cite{microsoft_PRV_2} by defining a coupling between $Y^L_i$ and $\tilde{Y}_i$ (conditioned on what we proved above regarding $\hat{\mu}$). Sample $y \sim Y^L_i$, and suppose $y \in \left(ih - \tfrac{h}{2}, ih + \tfrac{h}{2}\right]$ for some integer $-m \leq i \leq m$. Consider $\tilde{y} = \hat{\mu}+ih$. The distribution of $\tilde{y}$ matches that of $\tilde{Y}_i$ by construction in Algorithm~\ref{alg:sampleFunction}, so it is a valid coupling between $Y^L_i$ and $\tilde{Y}_i$ (because the two marginals $y,\tilde{y}$ have the intended distributions) and $|y-(\tilde{y}-\hat{\mu})| = |y-ih| \leq \tfrac{h}{2}$.

    Our second fact (analog of Lemma~C.6 in~\cite{microsoft_PRV_2}) is that given $|Y^L_i - (\tilde{Y}_i-\hat{\mu}_i)| \leq_0 \tfrac{h}{2}$ and $\mathbb{E}[\tilde{Y}_i] = \mathbb{E}[Y^L_i] + \tau_i$ for all $1 \leq i \leq k$ and for some $|\tau_i|\leq \gamma$,
    \[
        \left|\sum\limits_{i=1}^{k}{Y^L_i} - \sum\limits_{i=1}^{k}{\tilde{Y}_i}\right| \leq_{\eta'} h\sqrt{\frac{k}{2}\log\left(\frac{2}{\eta''}\right)} + \left|\sum\limits_{i=1}^{k}{\tau_i}\right|
    \]
    for some $\eta', \eta''\geq 0$.

    To prove this claim, let $Z_i = Y^L_i - \tilde{Y}_i$. Then $Z_i \in [-\hat{\mu}_i - \tfrac{h}{2}, -\hat{\mu}_i + \tfrac{h}{2}]$ and $\mathbb{E}[Z_i] = -\tau_i$. We now apply Hoeffding's inequality (Lemma~\ref{lem:hoeffding}) once again to get the following:
    \[
        \Pr\left[\left|\sum\limits_{i=1}^{k}{Z_i} - \sum\limits_{i=1}^{k}{(-\tau_i)}\right| \geq s\right] \leq 2e^{-\frac{2s^2}{kh^2}} \implies
        \Pr\left[\left|\sum\limits_{i=1}^{k}{Z_i}\right| \geq s + \left|\sum\limits_{i=1}^{k}{\tau_i}\right|\right] \leq 2e^{-\frac{2s^2}{kh^2}}.
    \]
    Setting $\eta' = 2e^{-\frac{2s^2}{kh^2}} + 2k(e^{-\frac{nt^2}{2L^2}} + e^{-\frac{nt^2}{2}})$ (by taking a union bound) and $\eta'' = 2e^{-\frac{2s^2}{kh^2}}$, we have the following:
    \[
        \left|\sum\limits_{i=1}^{k}{Y^L_i} - \sum\limits_{i=1}^{k}{\tilde{Y}_i}\right| \leq_{\eta'} h\sqrt{\frac{k}{2}\log\left(\frac{2}{\eta''}\right)} + \left|\sum\limits_{i=1}^{k}{\tau_i}\right|.
    \]
    This completes the proof of this claim.

    Now, we also know that
    \[
        \left|\sum\limits_{i=1}^{k}{Y^L_i} - \sum\limits_{i=1}^{k}{Y_i}\right| \leq_\lambda 0,
    \]
    where $\lambda = \sum_{i=1}^{k}{\Pr[|Y_i| \geq L]}$. Therefore, using Lemma~\ref{lem:triangle_inequality_couplings}, we get the following:
    \[
        \left|\sum\limits_{i=1}^{k}{Y_i} - \sum\limits_{i=1}^{k}{\tilde{Y}_i}\right| \leq_{\lambda+\eta'} s + \left|\sum\limits_{i=1}^{k}{\tau_i}\right|.
    \]
    Therefore, for $\eta_0 = \lambda + \eta'$ and $\tau_0 = s + k\gamma$, our lemma holds.
\end{proof}

\begin{proof}[Proof of Theorem~\ref{thm:sampled_prv_error}]
    We will first chain the results of Lemmas~\ref{lem:coupling-dist-to-disc},~\ref{lem:coupling-dist-to-samples}, and~\ref{lem:coupling-sample-to-conv} using Lemma~\ref{lem:triangle_inequality_couplings}, and we will show a coupling approximation between $\hat{Y}_n$ and $Y$, similar to the result of Corollary~C.8 in~\cite{microsoft_PRV_2}. Then we will apply Lemma~\ref{lem:coupling_to_privacy_curves} to provide guarantees on the privacy curves and finish our proof.

    From Lemma~\ref{lem:coupling-sample-to-conv}, we know that $|Y^n - \hat{Y}_n| \leq_{\eta_2} 0$, where $\eta_2 = \Pr\left[|\sum_{i=1}^{k}{\tilde{Y}_i}| \geq L\right] + 2k\alpha$. We can bound $\Pr\left[|\sum_{i=1}^{k}{\tilde{Y}_i}| \geq L\right]$ using Corollary~C.8 from~\cite{microsoft_PRV_2}, which says that,
    \[
        \Pr\left[\left|\sum_{i=1}^{k}{\tilde{Y}_i}\right| \geq L\right] \leq \Pr\left[\left|\sum\limits_{i=1}^{k}{Y_i}\right| \geq L-h\sqrt{\frac{k}{2}\log\left(\frac{2}{\eta''}\right)}\right] + \sum\limits_{i=1}^{k}{\Pr\left[|Y_i|\geq L\right]} + \eta',
    \]
    for some $\eta'>0$. From the proofs of Lemma~\ref{lem:coupling-dist-to-disc} and Corollary~C.8 from~\cite{microsoft_PRV_2}, we know that this $\eta' = 2e^{-\frac{2s^2}{kh^2}} + 2k(e^{-\frac{nt^2}{2L^2}} + e^{-\frac{nt^2}{2}})$ and $\eta'' = 2e^{-\frac{2t^2}{kh^2}}$. From Lemma~\ref{lem:sampling_coupling}, with probability at least $1-2ke^{-nt^2/2}$, the $\alpha$ in Lemma~\ref{lem:coupling-sample-to-conv} equals $t + \sqrt{\tfrac{L}{nh}}$. Applying Lemmata~\ref{lem:triangle_inequality_couplings},~\ref{lem:coupling-dist-to-disc},~\ref{lem:coupling-dist-to-samples}, and~\ref{lem:coupling-sample-to-conv}, and the union bound, we get
    \[
        \left|Y - \hat{Y}_n\right| \leq_\eta \tau,
    \]
    where
    \begin{align*}
        \eta = 2\sum_{i=1}^{k}{\Pr[|Y_i| \geq L]} + 4e^{-\frac{2s^2}{kh^2}} + 4ke^{-\frac{nt^2}{2L^2}} + 8ke^{-\frac{nt^2}{2}} + \Pr\left[|\sum\limits_{i=1}^{k}{Y_i}| \geq L-t\right] + 2k\left(t + \sqrt{\frac{L}{nh}}\right),
    \end{align*}
    and
\[
        \tau = s + k\left(t + 2L\left(\frac{t}{2} + \sqrt{\frac{L}{nh}}\right)\right) + 2k\left(\frac{t}{2} + \sqrt{\frac{L}{nh}}\right).
\]

    Finally, combining the above with Lemma~\ref{lem:coupling_to_privacy_curves}, for all $\epsilon \geq 0$,
	$$\delta_{\hat{Y}_n}(\eps+\tau)-\eta \le \delta_{Y}(\eps) \le \delta_{\hat{Y}_n}(\eps-\tau)+\eta.$$
    
    Our runtime guarantees come at a linear cost in $\tilde{O}(b\tfrac{L}{h} + bn)$ for computing the empirical histograms and the means. The first term comes from the computation of $\delta_{\hat{Y}_n}(\epsilon)$, which was shown in~\cite{microsoft_PRV_2} to take time $\tilde{O}(bL/h)$, and the second term comes from our sampling process. To elaborate on the second term, it takes $O(n)$ time to create a histogram for a single PRV, and there are $b$ distinct PRVs, so the added time for this step is just $O(bn)$. Computing $\hat{\mu}_i$ is takes $O(L/h)$ time, so the total time for all the $\hat{\mu}_i$'s is at most $O(bL/h)$, which gets absorbed in the first term.
\end{proof}

\subsection{Mechanisms with Equivalent Privacy Guarantees} \label{sec:mechanisms_with_equal_privacy}

For any privacy accountant, it is generally possible to run the accounting algorithm many times to compute the hyperparemeters required to achieve a particular degree of privacy. We introduce the following simple but effective algorithm for using the PRV accountant as part of a binary search over possible values of $\sigma$ in order to compute the minimal $\sigma$ value that $GG_{\beta,\sigma}$ satisfies $(\epsilon, \delta)$-DP for a given $\beta$. 

Let $PRV(\beta, \sigma, \delta)$ be a subroutine that runs the PRV accountant for the $GG_{\beta,\sigma}$ mechanism, and returns the $\epsilon$ value associated, such that $GG_{\beta,\sigma}$ satisfies $(\epsilon, \delta)$-DP.

\begin{algorithm}[H]
    \begin{algorithmic}[1]
        \State Input: $\beta \geq 1$, $\epsilon > 0$,  $\delta > 0$, tolerance $\tau >0$
        \State Output: $\sigma$, such that $GG_{\beta,\sigma}$ satisfies $(\epsilon, \delta)$-DP
        \State $\sigma_{min} = \sigma_{max} = 1$
        \While{$PRV(\beta, \sigma_{min}, \delta) > \epsilon$}
        \State $\sigma_{min} = \sigma_{min} /2.$
        \EndWhile
        \While{$PRV(\beta, \sigma_{max}, \delta) < \epsilon$}
        \State $\sigma_{max} = \sigma_{max} * 2.$
        \EndWhile
        \While{$PRV(\beta, \sigma_{max}, \delta) - \epsilon > \tau$}
        \State $\sigma_{mid} = \frac{\sigma_{max} + \sigma_{min}}{2}$
        \If{$PRV(\beta, \sigma_{mid}, \delta) > \epsilon$}
        \State $\sigma_{min} = \sigma_{mid}$
        \Else{ $\sigma_{max} = \sigma_{mid}$ }
        \EndIf 
        \EndWhile
        \State \textbf{return} $\sigma_{max}$
        \caption{Binary-search $\sigma$-solver}
    \end{algorithmic}
\end{algorithm} \label{algo:binary_search_sigma_solver}

For our empirical results in Sections~\ref{sec:argmax} and~\ref{sec:dpsgd}, we compare how the value of $\beta$ impacts accuracy for a fixed privacy guarantee.

\subsection{Outliers for Equivalently Private Mechanisms} \label{section:outliers_for_equivalent_mechanisms}

Using the sampled-PRV privacy accountant and the Binary-search $\sigma$-solver algorithm described in \Cref{sec:mechanisms_with_equal_privacy}, we are able to compute the $\sigma$-value --- as function of $\epsilon$, $\delta$, and $\beta$ --- such that $GG_{\beta, \sigma}(f,D)$ satisfies $(\epsilon, \delta)$-DP. Such a computation is currently not possible with other popular privacy accountants, such as the RDP accountant and GDP accountant,\footnote{These are the only two other privacy accountants supported by Opacus, the most popular private machine learning library \citet{opacus_paper}}, and no such analytic privacy bound currently exists for the GG mechanism for non-integer values of $\beta$. 

Combining this empirical privacy accountant with the known CDF of the Generalized Gaussian distribution $\mathcal{N}_\beta(0,\sigma)$ \citep{analytic_ggd}, we can compute the weight $w$ of the tail of the distribution, as a function of the mechanism's parameters and a cutoff parameter $\tau$, which specifies the threshold for defining the tail: $w(\beta, \epsilon,\delta, \tau) = \int_{-\infty}^{-\tau} GG_{\beta,\epsilon,\delta}(x) dx+\int_{\tau}^{\infty} GG_{\beta,\epsilon,\delta}(x) dx $, where $GG_{\beta,\epsilon,\delta}(x)$ is a probability distribution associated with a Generalized Gaussian for a particular value of $\beta$ that satisfies $(\epsilon,\delta)$-DP.\footnote{We are explicitly not specifying the value of $\sigma$ for ease of presentation. Specifying the input $\beta$ parameter and the desired privacy parameters $\epsilon,\delta$ will uniquely determine the minimum $\sigma$ value that achieves the privacy parameters.} Using this formalism, given a tail threshold $\tau$, any sample that falls in the tail can be labeled as an outlier, and the weight $w$ described how likely such an outlier is to be observed, given the mechanism's parameters.

\Cref{fig:optimal_mechanism} plots the weight in the tail of a GG distribution that satisfies a particular $(\epsilon,\delta)$-DP guarantee, as a function of $\beta$, and varying $\tau = \{1,2,4\}$. The left shows the tail weight for the GG Mechanism satisfying $(1.5, 10^{-5})$-DP, and the right shows tail weight for the SGG Mechanism satisfying $(1.5, 10^{-5})$-DP using Poisson sampling rate $q=0.01$ and composed over 100 rounds. Given a desired $(\epsilon,\delta)$-DP guarantee, a practitioner that wishes to minimize the probability of outliers in any of their computations could generate such a plot and then pick the choice $\beta$ which provides a minimum weight, thus minimizing the likelihood of outliers.

\begin{figure}[H]
    \centering
    \subfigure{\includegraphics[width=1.0\textwidth]{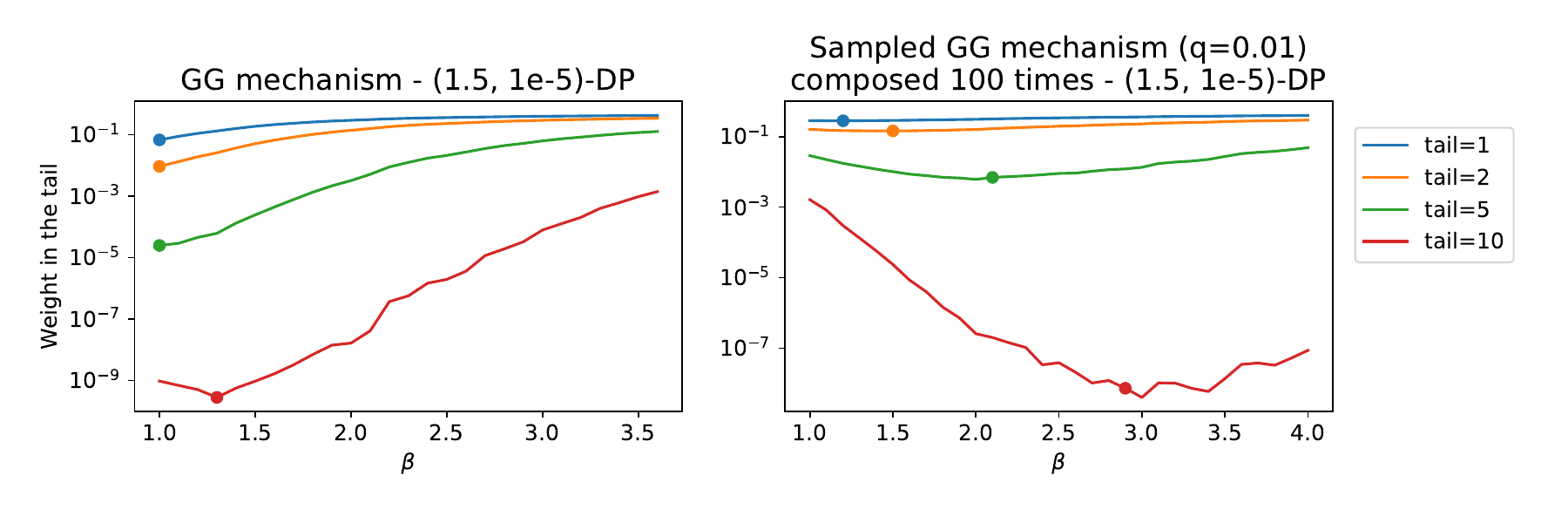}
    }
    \caption{Likelihood of outliers for the GG Mechanism  (left) and SGG Mechanism with Poisson sampling probability $q=0.01$ composed 100 rounds (right), both that satisfy $(1.5, 10^{-5})$-DP . In the legend, \emph{tail} refers to the $\tau$ tail cutoff parameter. The jaggedness in the plotting comes from the contributions of sampling error methods and integration error for very small weights; values are slightly smoothed with a Savitzky-Golay filter with polynomial order 2 and window size 5.
    }\label{fig:optimal_mechanism} 
\end{figure}

While for the single-shot GG mechanism (left), the Laplace mechanism appears optimal (i.e., has the least weight in the tail), \Cref{fig:optimal_mechanism} shows that that there are regimes where the tails of \Laplace~and \Gaussian~are heavier (i.e., outliers are more likely) than other equivalently private mechanisms in the GG family. This shows it is possible for that there to exist values of $\beta \notin \{1,2\}$ for which the associated GG Mechanism outperforms Laplace and Gaussian mechanism on this particular metric. This observation provides a potential direction for new methods to evaluate and search for alternative DP mechanisms: minimizing outliers was one of the main considerations cited by both the US Census \citep{census_TDA} and researchers behind PATE \citep{pate_2018} when deciding to use the \Gaussian~mechanism instead of the \Laplace~mechanism. Of course, the likelihood of outliers should be integrated with other organizational objectives and constraints, since in general, an algorithm designer's goal is not simply to minimize outliers, but rather to minimize some loss function (e.g., maximize accuracy) that is impacted by outliers

\section{Omitted Privacy Proofs}

\subsection{Bounded \Renyi Divergence of the Generalized Gaussian Mechanism} \label{appendix:rdp_of_gg_is_bounded} 

Here, we prove that the Renyi divergence between two Generalized Gaussian random variables is bounded, which is used in the proof of Theorem~\ref{thm.ggprivacy}. The two distributions in the statement of Lemma~\ref{lemma:renyi_div_is_bounded} correspond to the maximally different output distributions of the GG Mechanism on neighboring databases, where $\mu>0$ is the sensitivity of the input function. Note that this bound would be unchanged if we shifted the means of both distributions by the same amount, so we do not need to consider the values of the function, only the difference in the values.

\begin{lemma} \label{lemma:renyi_div_is_bounded}
    The Renyi Divergence $D_\alpha(\mathcal{N}_\beta(0,\sigma)|| \mathcal{N}_\beta(\mu,\sigma))$ is bounded for all $\alpha >1$, $\beta \geq 1$, $\sigma >0$, and $\mu > 0$.
\end{lemma}
\begin{proof}
    Fix $\beta\ge1$, $\sigma>0$, and $\mu>0$, and consider the PDFs of generalized Gaussian random variables, $\mathcal{N}_\beta(0,\sigma)$ and $\mathcal{N}_\beta(\mu,\sigma)$. These PDFs are:
    \[
    p(x)=\frac{1}{Z}\exp\Bigl(-\Bigl|\frac{x}{\sigma}\Bigr|^\beta\Bigr) \text{ and }
    q(x)=\frac{1}{Z}\exp\Bigl(-\Bigl|\frac{x-\mu}{\sigma}\Bigr|^\beta\Bigr), \] 
    with normalization constant $Z=2\sigma\,\Gamma\Bigl(1+\frac{1}{\beta}\Bigr)$.
    The Rényi divergence of order $\alpha>1$ is defined as,
    \begin{equation}\label{eq.reynidiv}
    D_\alpha(p\|q)=\frac{1}{\alpha-1}\log\int_{-\infty}^{\infty}\frac{p(x)^\alpha}{q(x)^{\alpha-1}}\,dx.
     \end{equation}
    Substituting the densities into the integrand gives:
    \begin{align*}
    \frac{p(x)^\alpha}{q(x)^{\alpha-1}}
    &=\frac{1}{Z^\alpha}\exp\Bigl(-\alpha\Bigl|\frac{x}{\sigma}\Bigr|^\beta\Bigr)
    \cdot Z^{\alpha-1}\exp\Bigl((\alpha-1)\Bigl|\frac{x-\mu}{\sigma}\Bigr|^\beta\Bigr)\\
    &=\frac{1}{Z}\exp\Bigl(-\alpha\Bigl|\frac{x}{\sigma}\Bigr|^\beta+(\alpha-1)\Bigl|\frac{x-\mu}{\sigma}\Bigr|^\beta\Bigr).
    \end{align*}
    
    Define $f: \mathbb{R} \to \mathbb{R}$ as,
    $$f(x) = \exp\left(-\alpha\left|\frac{x}{\sigma}\right|^\beta+(\alpha-1)\left|\frac{x-\mu}{\sigma}\right|^\beta\right).$$
    It is sufficient to show that, barring a finite range of real numbers, the exponent in $f$ is negative, which would make the integral converge. For that, we consider three intervals: $x < 0$, $0 \leq x \leq \mu$, and $\mu < x$, and define three functions $f_1$, $f_2$, and $f_3$ on these intervals, respectively, as follows.
    \begin{align*}
        f_1(x) &\coloneqq -\alpha\left(-\frac{x}{\sigma}\right)^{\beta} + (\alpha-1)\left(\frac{\mu-x}{\sigma}\right)^{\beta}\\
        f_2(x) &\coloneqq -\alpha\left(\frac{x}{\sigma}\right)^{\beta} + (\alpha-1)\left(\frac{\mu-x}{\sigma}\right)^{\beta}\\
        f_3(x) &\coloneqq -\alpha\left(\frac{x}{\sigma}\right)^{\beta} + (\alpha-1)\left(\frac{x-\mu}{\sigma}\right)^{\beta}
    \end{align*}
    Note that these functions describe the behavior of the exponent of $f$ in their respective domains.

    We start by showing that $f_3$ is always negative in the range $(\mu,\infty)$. Note that the right derivative of $f_3$ equals that of its left derivative at $x=\mu$ because it is a polynomial function and the limits on both the sides will be equal. We compute the derivative of $f_3$ as follows:
    $$f'_3(x) = -\frac{\alpha\beta x^{\beta-1}}{\sigma^{\beta}} + \frac{(\alpha-1)\beta(x-\mu)^{\beta-1}}{\sigma^{\beta}}.$$
    For any $t>0$, $x^t$ is an increasing function, but on $t=0$, $x^t=1$ is a constant function. Therefore, for $\beta=1$,
    $$f'_3(x) = -\frac{\alpha\beta}{\sigma^{\beta}} + \frac{(\alpha-1)\beta}{\sigma^{\beta}} = -\frac{\beta}{\sigma^{\beta}} < 0.$$
    For $\beta>1$,
    \begin{align*}
        f'_3(x) = -\frac{\alpha\beta x^{\beta-1}}{\sigma^{\beta}} + \frac{(\alpha-1)\beta(x-\mu)^{\beta-1}}{\sigma^{\beta}}
        = \frac{\beta}{\sigma^{\beta}}\left[-\alpha x^{-\beta-1} + (\alpha-1)(x-\mu)^{\beta-1}\right].
    \end{align*}
    However, since $x>\mu$, $x^{\beta-1} > (x-\mu)^{\beta-1}$ and $\alpha > 1$, $f'_3$ must always be negative whenever $x>\mu$. Also, $f_3(\mu) = -\alpha\left(\tfrac{\mu}{\sigma}\right)^{\beta} < 0$, which implies that $f_3(x) < 0$ whenever $x > \mu$.

    We next consider the range $0 \leq x \leq \mu$. Not that this is a finite interval and that $f_2(x)$ will be finite on all $x$ in this domain because $f_2$ is a polynomial function where the rational terms are well-defined in the domain.

    We finally consider the third range $x < 0$ and the function $f_1$. Substitute $y=-x$ and define $g_1(y) = \alpha\left(\tfrac{y}{\sigma}\right)^{\beta} + (\alpha-1)\left(\tfrac{y+\mu}{\sigma}\right)^{\beta}$ for $y>0$. In this, we have simply changed the variables and the domains so that $f_1(x) = g_1(-x)$ for all $x < 0$. Now, the argument for this case is the same as in that for $f_3$, while noting that the term $\left(\tfrac{y}{\sigma}\right)^{\beta}$ is now shifted to the left as $\left(\tfrac{y+\mu}{\sigma}\right)^{\beta}$. Therefore, $g_1(y) < 0$ for $y>0$, which implies that $f_1(x) < 0$ for $x < 0$.

    Finally, we can split the integral for the Renyi divergence into three intervals $x < 0$, $0 \leq x \leq \mu$, and $\mu < x$. We have shown that in the middle interval, the integrand is bounded, therefore, the integral will be bounded in this finite range. On the other hand, we have shown that in the other two intervals, the exponent in $f$ is always negative, which means that the integral in both cases will converge to finite values. This proves the lemma.
\end{proof}

\subsection{DP Guarantees of Generalized Gaussian Mechanism}  \label{appendix:gg_is_private}

\ggprivacy*
\begin{proof}
    By \Cref{lemma:renyi_div_is_bounded} we know that the \Renyi divergence of the GG Mechanism on any two neighboring databases $D_\alpha(\mathcal{N}_\beta(0,\sigma)|| \mathcal{N}_\beta(\Delta f,\sigma))$ is bounded. By the definition of \Renyi Differential Privacy \citep{RDP_2017} (\Cref{def:RDP-calculation}), then the GG Mechanism must satisfy $(\alpha, \epsilon)$-RDP for the finite $\epsilon$ corresponding to the upper bound on $D_\alpha(\mathcal{N}_\beta(0,\sigma)|| \mathcal{N}_\beta(\Delta f,\sigma))$. Then by the known translation between RDP and DP \citep{RDP_2017}, the GG Mechanism also satisfies $(\epsilon + \frac{\log (1/ \delta)}{\alpha -1}, \delta)$-DP for the same $\epsilon$ and for any $\delta \in (0,1)$.   
\end{proof}

\subsection{DP Guarantees of GGNmax}\label{appendix:private_argmax_is_private} \label{sec:ggnmax_priv_proof}

\ggnmaxIsPrivate* 

The proof of Theorem \ref{thm.ggargmaxpriv} follows closely to the proof of privacy of the ReportNoisyMax algorithm with Laplace noise, as presented in \citet{privacy_book}. It is included here for completeness.

\begin{proof} 
    Fix neighboring databases $D,D'$, where $D = D' \cup \{a\}$. Let $c$ and $c'$ respectively denote the vector of function values (i.e., counts) when the database is $D$ and $D'$. We use two properties:
    \begin{enumerate}
        \item \textit{Monotonicity of counts}. For all $j \in [N]$, $c_j \geq c'_j$
        \item \textit{Lipschitz property}. For all $j \in [N], 1 + c'_j \geq c_j$
    \end{enumerate}

    Fix any $i \in [N]$. We will bound from above and below ratio of the probabilities that $i$ is selected with $D$ and with $D'$. Fix $r_{-i}$, a draw from $[\mathcal{N}_{\beta}(0, \sigma)]^{N-1}$ used for all the noisy function evaluation except for the $i$th function. We use the notation $Pr[i| \zeta]$ to mean the probability that the output of the GGNmax algorithm is $i$ conditioned on $\zeta$.

    We first argue that $\Pr[i|D, r_{-i}] \leq e^{\epsilon} \Pr[i|D', r_{-i}] + \delta$. Define 
    \[
        r^* = \min_{r_i}: c_i + r_i > c_j + r_j \; \; \forall j \neq i.
    \]
    Note that having fixed $r_{-i}$, $i$ is will be the output of the GGNmax mechanism (i.e., the $i$th function will have the largest noisy count) when the database is $D$ if and only if $r_i\geq r^*$. Then for all $1 \leq j \neq i \leq N$: 
    \begin{align*}
            c_i + r^* &> c_j + r_j \\
            \Rightarrow (1 +c'_i) + r^* \geq c_i + r^* &> c_j + r_j \geq c'_j + r_j \\
            \Rightarrow c'_i + (r^* + 1) &> c'_j + r_j    
    \end{align*}

    Thus if $r_i \geq r^* + 1$, then $i$ will be the output of GGNmax, meaning that the $i$th function had the largest noisy count, when the database is $D'$ and the noise vector is $(r_i, r_{-i})$. We now wish to relate the probability of $\Pr[r_i \geq 1 + r^*]$ to the probability $ \Pr[r_i \geq r^*]$.  We note that this is the step where the proof diverges from the privacy proof of Report Noisy Max, as given in \cite{privacy_book}.

  Let $Z \sim \mathcal{N}_{\beta}(0, \sigma)$. We are given that $GG_{\beta, \sigma}(f,D)$ satisfies $(\epsilon, \delta)$-DP, so,
    \begin{align*}
       \Pr[Z \geq r^*] &\leq e^{\epsilon} \Pr[Z \geq r^* + 1] + \delta \\ 
        \Rightarrow  \Pr[i| D, x_{-i}] = \Pr[r_i \geq r^*] &\leq e^{\epsilon} \Pr[r_i \geq r^* + 1] + \delta \leq e^\epsilon \Pr[i|D', x_{-i}] + \delta.
    \end{align*}
    
    We now argue that $Pr[i|D'] \leq e^{\epsilon} \Pr[i|D] + \delta$. Define, again,
    \[
        r^* = \min_{r_i}: c'_i + r_i > c'_j + r_j \; \; \forall j \neq i.
    \]

    Note that having fixed $r_{-i}$, $i$ is will be the output of the GGNmax when the database is $D'$ if and only if $r_i\geq r^*$. For all $1 \leq j \neq i \leq N$,
    \begin{align*}
            c'_i + r^* &> c'_j + r_j \\
            \Rightarrow 1 + c'_i + r^* &> 1 +  c'_j + r_j \\
            \Rightarrow c'_i + (1 + r^*) &> (1 + c'_j) + r_j \\
            \Rightarrow c_i + (1+ r^*) \geq c'_i + (r^* + 1) &> (1 + c'_j) + r_j \geq c_j + r_j 
     \end{align*}
     
    Thus, if $r_i \geq r^* + 1$, then $i$ will be the output of GGNmax on database $D$ with randomness $(r_i, r_{-i})$. We again wish to relate the probability of $\Pr[r_i \geq 1+ r^*]$ to the probability $ \Pr[r_i \geq r^*]$, where $Z \sim \mathcal{N}_{\beta}(0, \sigma)$. Again using the fact that $GG_{\beta, \sigma}(f,D)$ satisfies $(\epsilon, \delta)$-DP, 
     \begin{align*}
        \Pr[Z \geq r^*] &\leq e^{\epsilon} \Pr[Z \geq r^* + 1 ] + \delta \\ 
        \Rightarrow  \Pr[i| D', x_{-i}] = \Pr[r_i \geq r^*] & \leq e^{\epsilon} \Pr[r_i \geq r^* + 1] + \delta \leq e^\epsilon \Pr[i|D, x_{-i}] + \delta.
    \end{align*}

    Thus the GGNmax mechanism satisfies the same $(\epsilon,\delta)$-DP guarantee as the 1-dimensional $GG_{\beta,\sigma}(f,D)$ mechanism with sensitivity $\Delta=1$ for the count functions. 
\end{proof}

\section{Additional Details and Results for Private Argmax and PATE}\label{app.argmas} \label{appendix:argmax_and_pate}

PATE (Private Aggregation of Teacher Ensembles) is an algorithm to train a private machine learning model. In the first step, the private dataset is partitioned into $T$ datasets, such that a single user's data is only in single partition. A ``teacher'' model is trained for each partition. Then, the teacher models are collected to privately vote on how to label an unlabeled, public dataset, usually through an algorithm based off of the Report-Noisy-Max algorithm. A ``student'' model is trained on the privately labeled dataset. We present a high-level overview of the algorithm in \Cref{fig:pate_diagram}.

\begin{figure}[tbh]
    \centering
    \subfigure{\includegraphics[width=0.95\textwidth]{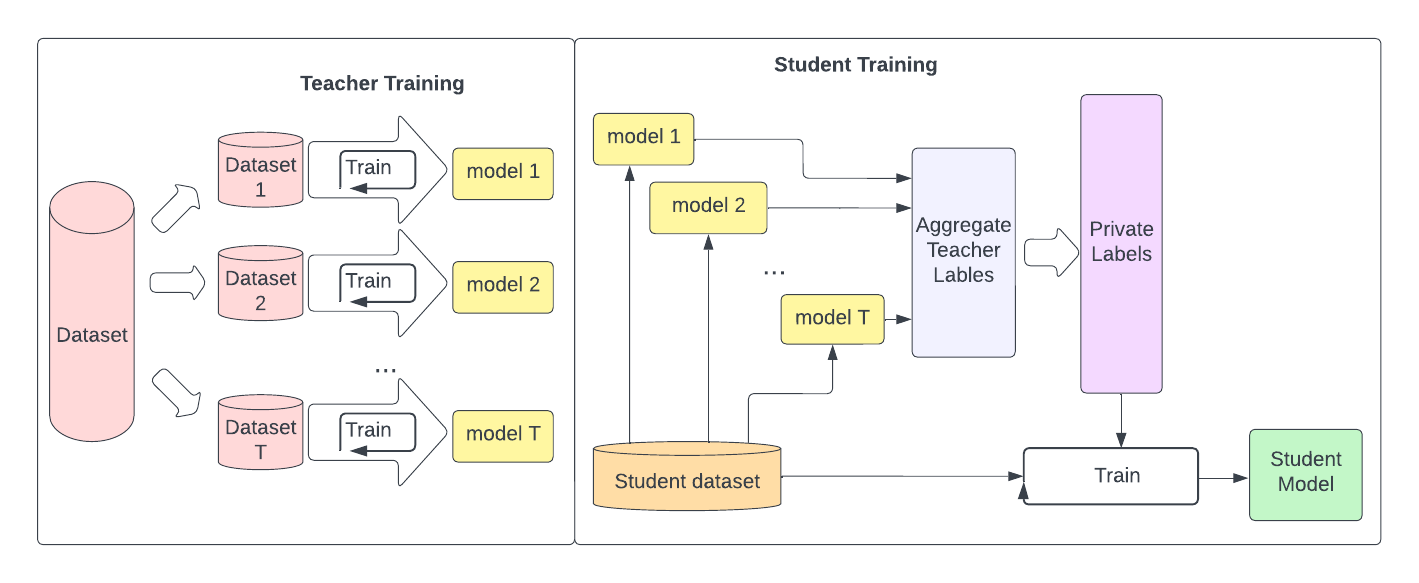}}
    \caption{\small{Diagram of PATE implementation \citep{pate_2017}}} 
    \label{fig:pate_diagram}
\end{figure}

In the main body of the paper we primarily explore the idea of how the choice of $\beta$ affects the label accuracy of PATE; in order to explore this more fully, we propose a new measure for accuracy for the private \Argmax\ problem, which is better suited to the goals of private ML --- measuring the probability of returning the true \Argmax, rather than returning an outcome with score similar to the true \Argmax. For this utility measure, we empirically find that $\beta=2$ (\Gaussian) is again near-optimal. 

\subsection{Simulations for Ensemble-Based Private Vote Aggregation} \label{sec:argmax_simulations}

While PATE is one of the primary motivations for the GGNMax mechanism, taking a private \Argmax\ is a very general problem and is particularly important for algorithms which attempt to reconcile beliefs (or votes) across many parties. Classical work in differential privacy on private \Argmax\ considers that for a vector of values, the utility of the mechanism is a function of the probability that the mechanism returns any index that has a value associated with, \textit{close} to the maximum value. However, in a task like classification in ML, there is only a single label that is the ``correct'' label, and thus we argue that for a task like classification in ML, a mechanism should be evaluated on how often it returns the label that would have been assigned without noise.

Building upon this intuition, we define the \emph{Hardmax Utility} of an \Argmax\ mechanism $\mathcal{M}$ on functions $\{f_i\}$ over a distribution $\mathcal{P}$ of databases as:
\[
    \textrm{Hardmax-Utility}_{\mathcal{P}}(\mathcal{M}, \{ f_i\}) := \Pr_{D \sim \mathcal{P}}[\mathcal{M}(D, \{f_i\}) = \textstyle\arg\max_i(f_i(D)) ].
\]

With this utility measure in mind, we wish to measure the impact of $\beta$ on the Hardmax utility of GG Private \Argmax\ algorithm. Given a vector of function values $\{f_i(D)\}$, noise addition can only change the \Argmax\ if the noise added is larger than the existing gap between the highest function value and all other values. We refer to the difference between the largest and second largest value in $\{f_i(D)\}$, as the \emph{runner-up-gap}. 

To empirically evaluate the impact of $\beta$ on the Hardmax utility, we construct a set of random histograms with varying running-up-gaps; by varying the runner-up-gap, we vary the extend to which the outcome of the mechanism can be changed by outliers in the noise terms.

For our simulations, we construct 500 histograms of votes as follows. We fix the total number of votes (simulating votes from teacher models) as $V=1000$ and fix the largest number of votes for class 0 as $x_0 = v$. For a runner-up parameter $r\in[.001, 0.2]$, we fix the largest number of votes for class 0 at $x_0 = v$, and the second-largest number of votes for class 1 at $x_1 = v (1-r)$, corresponding to a runner-up-gap of $vr = 100r$. For 2-class histograms, fixing $V$ and $r$ fully specifies the (deterministic) histogram. 
For $N$-class histograms, votes for the first three classes $x_0, x_1, x_2$ are set so $x_1 = x_0 \cdot (1-r)$ (to ensure a runner-up-gap of $r$), $x_2 = x_1\cdot 0.95$, and $\frac{(N-3)*x_2}{2} + x_1+x_0 = V$. The next $N-4$ classes have votes uniformly sampled on the integers between $0$ and $x_2$, inclusive, and the final class has the remainder of votes required to reach $V$ votes in total: $x_{N-1} = V- \sum_{i=0}^{N-2} x_i$. If $x_{N-1}$ is negative or greater than $x_1$, the histogram is not used and random values are resampled. We then instantiate GGNmax on each database (histogram) with counting queries $f_i$ that count the number of votes for each class $i$.

For a given $(\epsilon, \delta)$-DP guarantee and for a range of values  $\beta_i \in [1,4]$, we compute $\sigma_i$ such that $(\beta_i, \sigma_i$)-Generalized Gaussian Private \Argmax\ satisfies $(\epsilon, \delta)$-DP (see Appendix \ref{sec:mechanisms_with_equal_privacy} for algorithmic details of this process). For each pair $(\beta_i, \sigma_i)$, we compute the Hardmax Utility of each mechanism by computing the likelihood of the $(\beta_i, \sigma_i)$-Generalized Gaussian Private Argmax algorithm returning the true argmax, averaged across 50 trials.

\begin{figure}[tbh]
    \centering
    \subfigure{\includegraphics[width=0.85\textwidth]{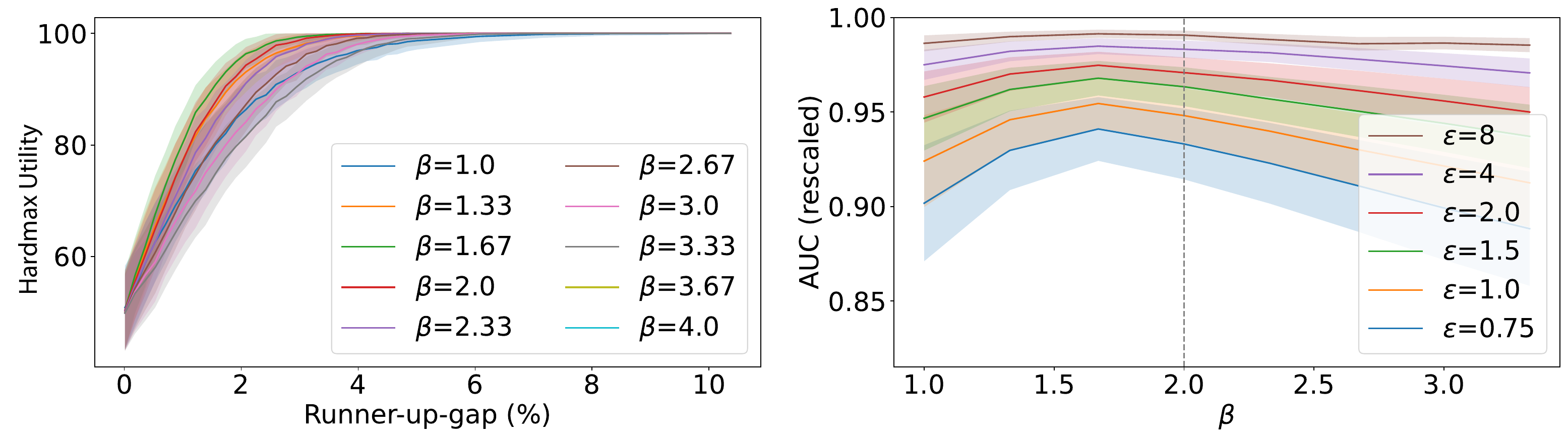}}
    \caption{\small{Hardmax Utility for 2-class histogram. Left: Hardmax Utility as a function of runner-up-gap, for mechanisms satisfying (1,$10^{-5}$)-DP. Right: Area-Under-the-Curve (AUC) of the curves on the left for different values of $\epsilon$. AUC is rescaled so that 1.0 is the maximum. 
    }
    }
    \label{fig:probability_of_argmax_AUC}
\end{figure}

\Cref{fig:probability_of_argmax_AUC} shows the Hardmax Utility of $(\beta_i, \sigma_i$)-Generalized Gaussian Private \Argmax\ for 2-class histograms
under varying $\beta$ values. The left shows the Hardmax utility as a function of the runner up gap for varying $\beta$ values. We observe that varying $\beta$ for a fixed runner-up-gap does not have a substantial impact on Hardmax utility, and that the optimal $\beta$ value typically does not dependent on the runner-up-gap. 

The right shows the Area-Under-the-Curve (AUC) of the Hardmax utility curves (similar to those on the left) as a function of $\beta$ for different values of $\epsilon$. Higher AUC means a better overall accuracy. The AUC is integrated over a range of runner-up-gap values from 0\% to 10\%, and AUC is rescaled such that 1.0 is the maximum. We observe that empirically, values of $\beta$ close to $\beta=2$ have better Hardmax-Utility than ones further away, regardless of the $(\epsilon, \delta)$-DP parameters. However, we do note that while $\beta =2$ is near-optimal, $\beta$ values slightly smaller than 2 do give even better performance. A key takeaway is that \Gaussian~does outperform \Laplace, although there is room for further improvements over \Gaussian\ by fine-tuning the $\beta$ parameter.

\Cref{fig:probability_of_argmax_AUC_multiclass} extends these findings to a multi-class setting with 25 classes. On the left, we see that the choice of $\beta$ has a larger impact on Hardmax utility in the multi-class setting, but that the optimal choice $\beta$ is again independent of the runner-up-gap. On the right, we observe that the impact of $\beta$ is also more pronounced in this multi-class setting, and that the optimal $\beta$ value is much closer to 2 in this setting. 

We also observe that as $\beta$ grows, the AUC (right) becomes more jagged -- this is also true for Hardmax Utility (left), but is more difficult to observe. This is because for a fixed $\epsilon$, the sensitivity of $\beta$ on $\sigma$ increases as $\beta$ grows. Let $\sigma'(\beta, \epsilon, \delta)$ be the function that returns the minimum value of $\sigma$ that satisfies $(\epsilon,\delta)$-DP for a given choice of $\beta$; the value of $\sigma'$ increases in $\beta$ at an increasing rate for a fixed $(\epsilon, \delta)$ pair, as can be observed in Figure \ref{fig:epsilon_dp_as_a_function_of_beta}. Our experiments use evenly spaced values of $\sigma$ and $\beta$, so the gap between the truly optimal $\sigma$ and the best choice of $\sigma$ from an evenly discretized set, increases as a function of $\beta$.

\begin{figure}[h]
    \centering
\subfigure{\includegraphics[width=0.85\textwidth]{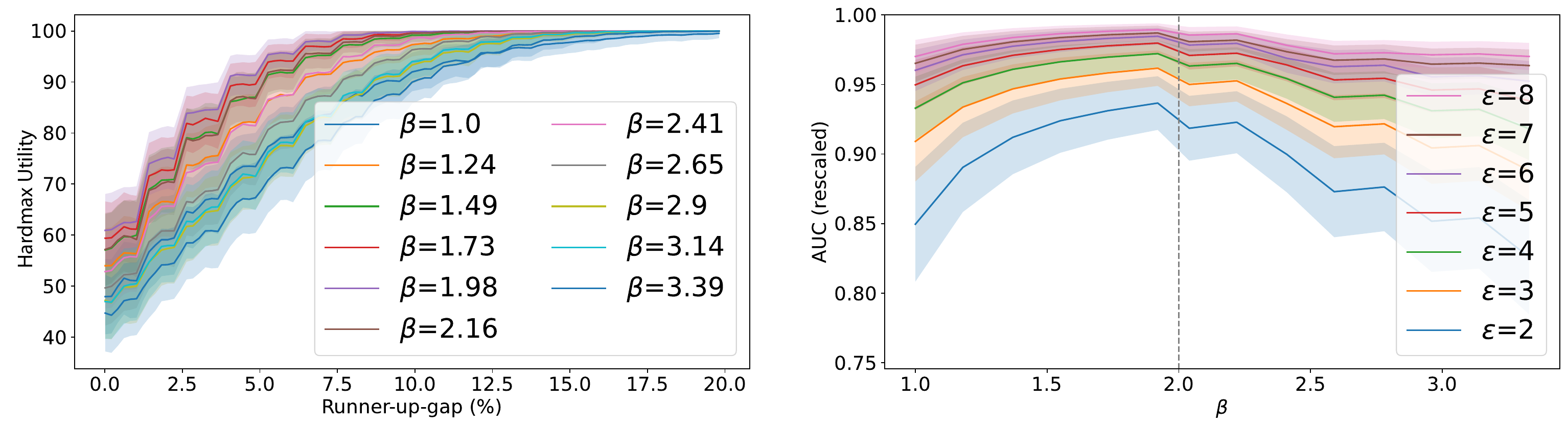}}
    \caption{\small{Hardmax Utility for 25-class histogram. Left: Hardmax Utility as a function of runner-up-gap, for mechanisms with equivalent (2,$10^{-5}$)-DP. Right: Area-Under-the-Curve (AUC) of the curves on the left for different values of $\epsilon$, where a higher AUC means a better overall accuracy. AUC is rescaled such that 1.0 is the maximum to normalize across different sizes of domain.}}
    \label{fig:probability_of_argmax_AUC_multiclass}
\end{figure}

\section{Additional $\beta$-DP-SGD Results and Implementation Details}\label{sec:dpsgd_results}

\subsection{Hyperparameters Used in Training} \label{appendix:hyperparameters_in_dpsgd}

\textbf{Hyperparameters} We run our $\beta$-DP-SGD algorithm for a maximum of 100 epochs for each parameter setting and sweep over the following parameters: $\beta \in \{1.0, 1.33, 1.66, 2.0\}$, learning rate $\eta\in\{0.5, 1.0\}$, clipping norm $C \in \{0.05, 0.1, 0.25, 0.5\}$, and noise multiplier $\sigma \in \{2.0, 2.5, 3.0\}$, for $\delta=10^{-6}$.  In practice, we run for $T=100$ epochs, then compute the privacy budget after each training epoch. We return the maximum utility the model achieved before the privacy budget was exceeded; we also early stop if the utility drops for 10 epochs consecutively. Each experiment is run 3 times, which we found sufficient given standard deviations that generally fell below 0.3\%.

\textbf{Datasets:} We train on CIFAR-10 \citep{cifar_10} and Street View House Numbers (SVHN) \citep{svhn}, two common computer vision datasets, which respectively contain 60,000 and 99,289 small, color images split across ten classes; the Adult dataset \citep{adult_dataset}, a tabular dataset with a binary classification task; and the IMDB dataset \citep{imdb_dataset}, a collection of movie reviews meant for binary sentiment classification.

\textbf{Models:} For the vision classification tasks (CIFAR-10 and SVHN), we use the models described in in \cite{handcrafted_dp}, which previously achieved SOTA results for the $\epsilon \leq \sim 2.5$ regime. Specifically, we train Convolutional Neural Networks (CNNs) on pretrained image features produced by scattering networks \cite{scatternet} as described in \citet{handcrafted_dp}; these are referred to as \emph{ScatterNet}. For the the Adult Dataset we train a 2-layer Fully Connected Network (FCN), with 32 neurons in the hidden layer. For the IMDB dataset, we train a Long-Short Term Memory (LSTM) network with 1,081,002 parameters, in order to demonstrate the method on a relatively medium-sized model from scratch.

\subsection{The Role of Individual Hyperparameters in $\beta$-\DPSGD} \label{app:hyperparams_in_dpsgd} \label{appendix:extra_dpsgd_results}

Below, we investigate the role of individual hyperparameters on the the final test-accuracy. We replicate the experiments of Section \ref{sec:dpsgd_experiments}, but holding fixed individual hyperparameters, and varying those fixed values to determine whether that hyperparameter substantially impacts test-accuracy. Specifically, we evaluate the impact of the learning rate (\ref{app.learningrate}), clipping norm (\ref{app.clipnorm}), and noise multiplier (\ref{app.noisemult}). The remaining parameters are optimized as described in Appendix \ref{appendix:hyperparameters_in_dpsgd}, and we see little effect in varying other values. We report the average maximum test-accuracy for each trial independently -- this means that for each plot and each fixed value of $\beta$, hyperparameters are evaluated independently, and the maximum average-test-accuracy is taken for each value of $\beta$.   

In general, we observe that while some hyperparameters may have some small effect, varying these hyperparameters do not have a substantial effect on test-accuracy. They also do not impact the general relationship between $\beta$ and test-accuracy: within the range $\beta \in [1,2]$ we found that, $\beta=2$ performs as well as or better than other values in that range for most parameter settings.

As in the presentation of results in Section \ref{sec:dpsgd_experiments}, some plots do not have test-accuracy values reported for specific choices of $\beta$ and $\epsilon$. This is because larger values of $\beta$ tend to consume more privacy per-step, and the starting $\epsilon$ required exceeded the $\epsilon$ budget for that curve.

\subsubsection{Learning Rate}\label{app.learningrate}

Here we explore the impact of varying the learning rate on the test-accuracy, as a function of $\beta$. Figures \ref{fig:dpsgd_results_scatter_separate__learning_rate} and \ref{fig:dpsgd_results_scatter_separate__learning_rate2} show test-accuracy results of $\beta$-\DPSGD\ on all four databases with respective learning rates 0.5 and 1.0. We observe that performance is relatively similar across these two figures, and thus conclude that learning rate does not have a strong impact.

    \begin{figure}[H]
        \centering
        \subfigure{\includegraphics[width=.8\textwidth]{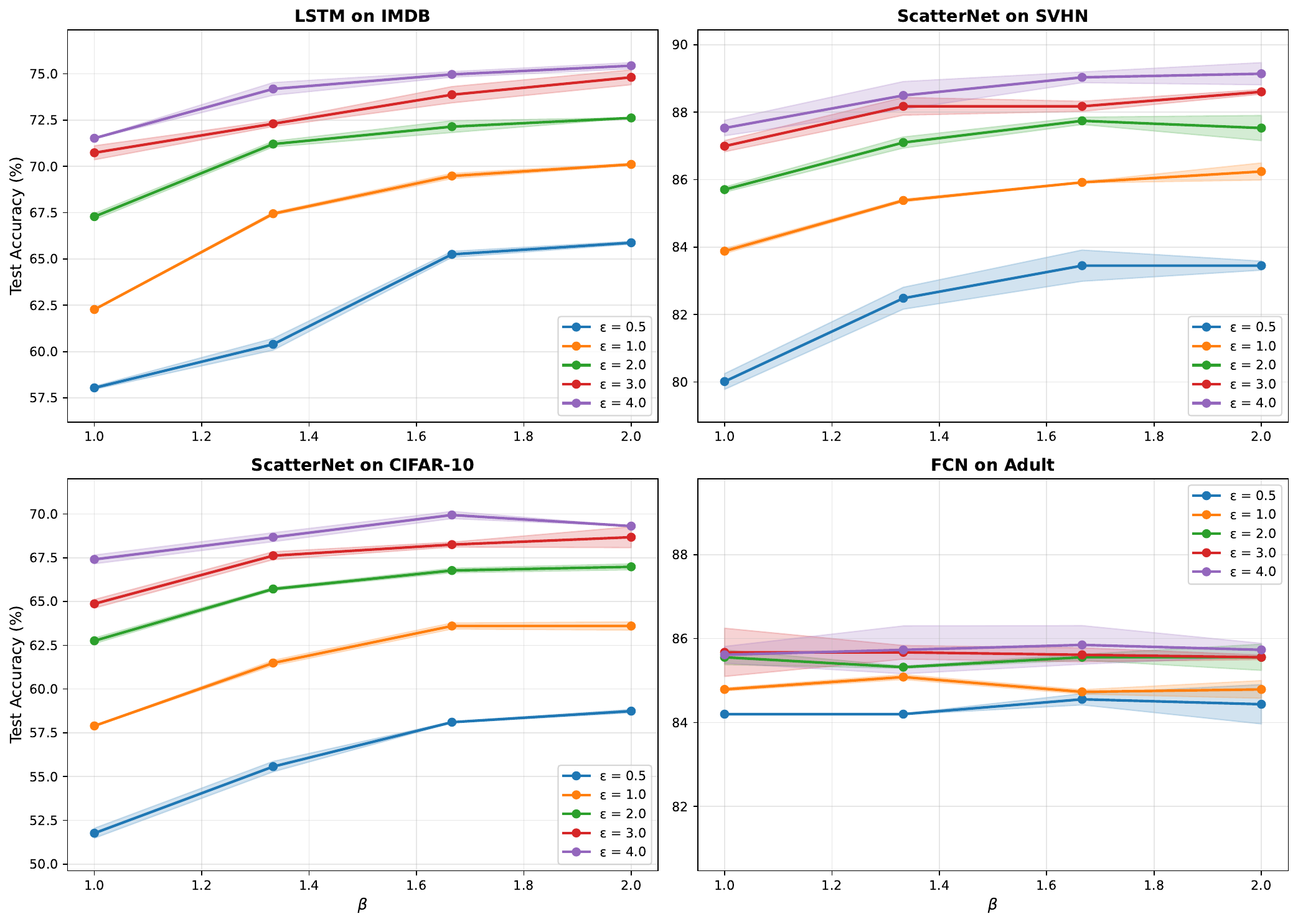}}
         \caption{\small{$\beta$-\DPSGD~results with learning rate 0.5}}
        \label{fig:dpsgd_results_scatter_separate__learning_rate}
    \end{figure}

        \begin{figure}[H]
        \centering
        \subfigure{\includegraphics[width=.8\textwidth]{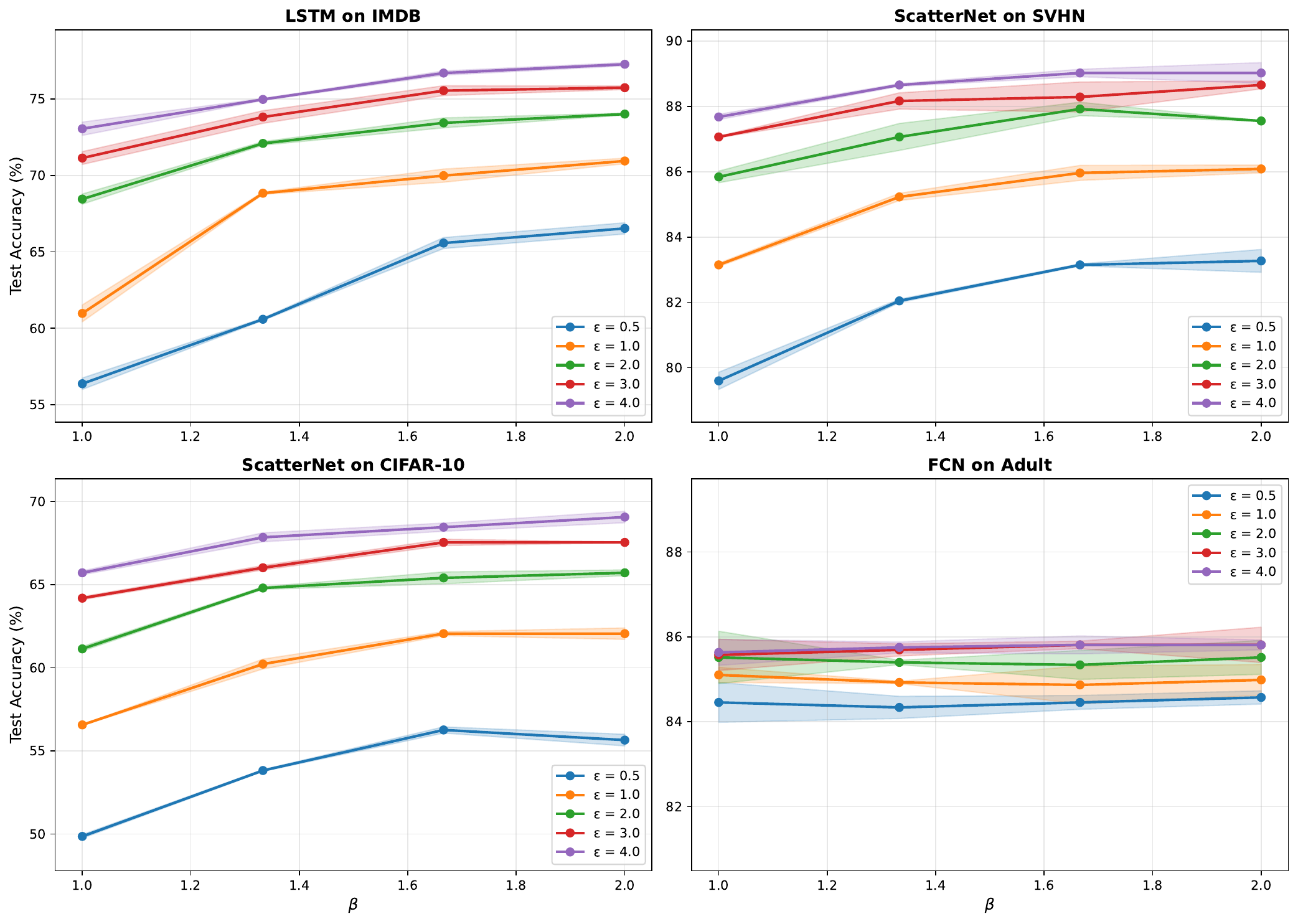}}
        \caption{\small{$\beta$-\DPSGD~results with learning rate 1.0}}
        \label{fig:dpsgd_results_scatter_separate__learning_rate2}
    \end{figure}

\subsubsection{Clipping Norm}\label{app.clipnorm}

Here we explore the impact of varying the clipping norm on the test-accuracy, as a function of $\beta$. Figures \ref{fig:dpsgd_results_scatter_separate__clipping}, \ref{fig:dpsgd_results_scatter_separate__clipping2}, and \ref{fig:dpsgd_results_scatter_separate__clipping3} show test-accuracy results of $\beta$-\DPSGD\ on all four databases with respective clipping norms of 0.05, 0.10, and 0.25. We observe that clipping norm appears to have a bigger effect on the relationship between test-accuracy and $\beta$ than learning rate, but still relatively minor and only at larger (sub-optimal) values of $\beta$.

    \begin{figure}[H]
        \centering
\subfigure{\includegraphics[width=.8\textwidth]{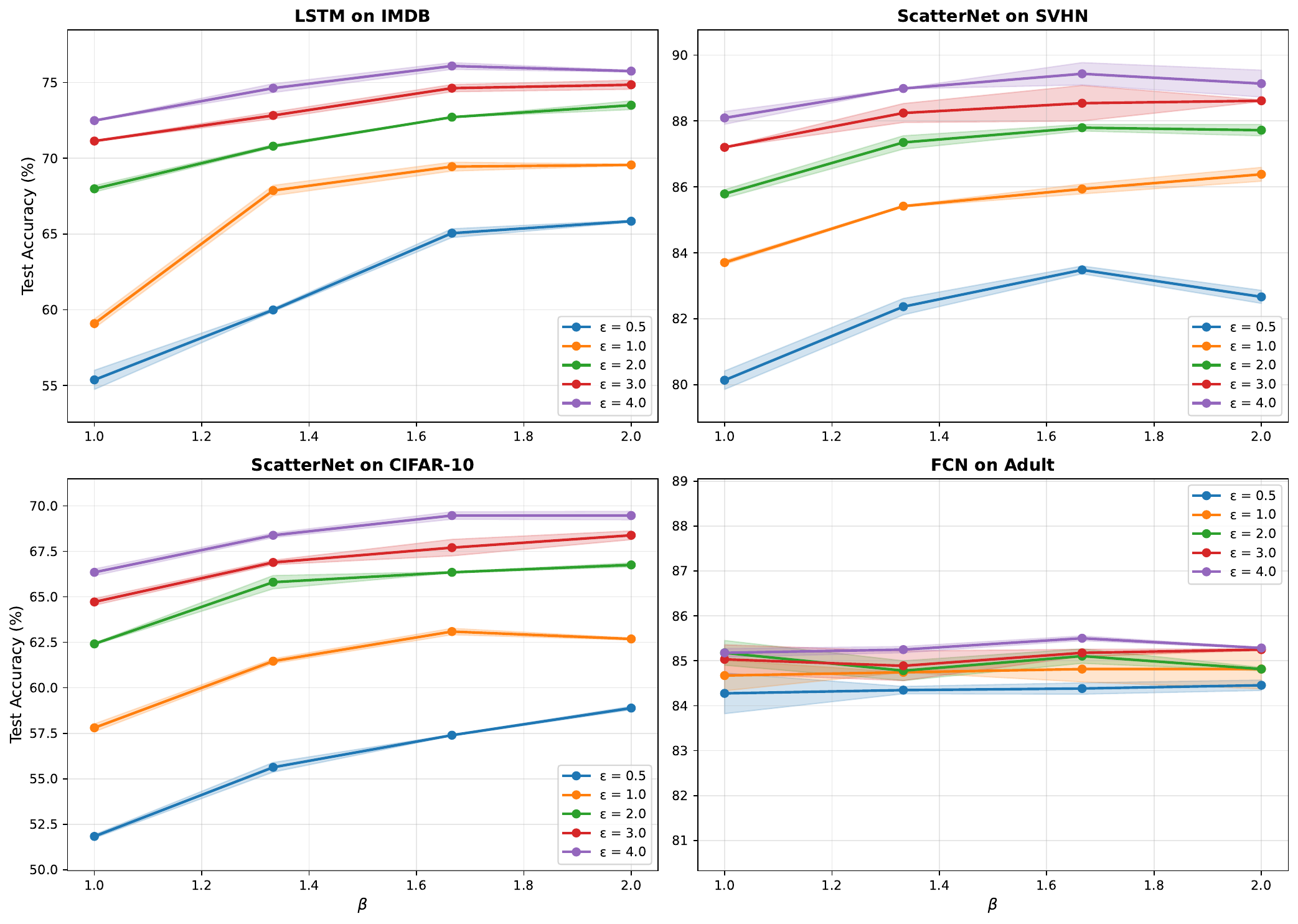}}
        \caption{\small{$\beta$-\DPSGD~results with clipping norm 0.05}}
        \label{fig:dpsgd_results_scatter_separate__clipping}
    \end{figure}

        \begin{figure}[H]
        \centering
\subfigure{\includegraphics[width=.8\textwidth]{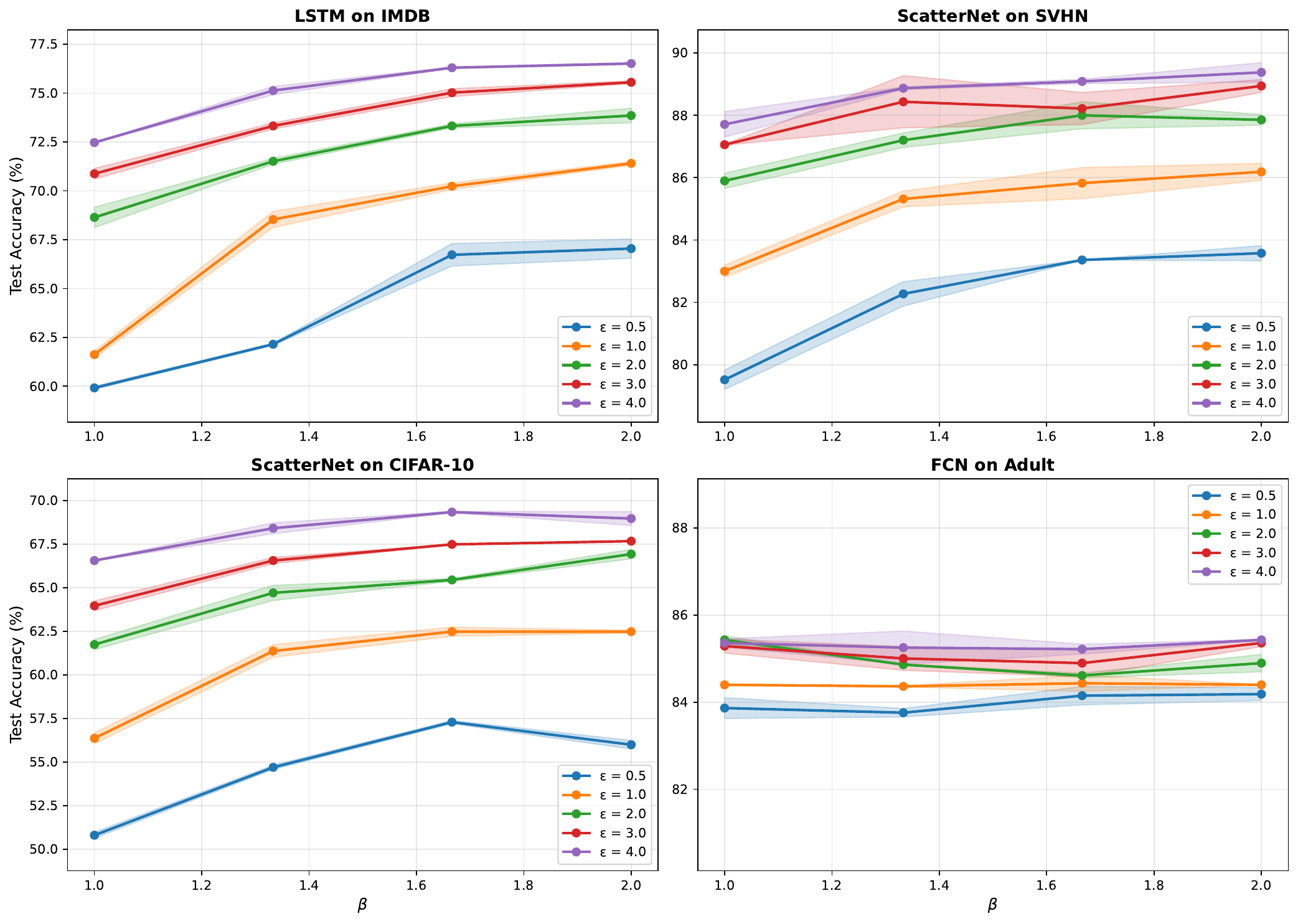}}
        \caption{\small{$\beta$-\DPSGD~results with clipping norm 0.1}}
        \label{fig:dpsgd_results_scatter_separate__clipping2}
    \end{figure}

        \begin{figure}[H]
        \centering
\subfigure{\includegraphics[width=.8\textwidth]{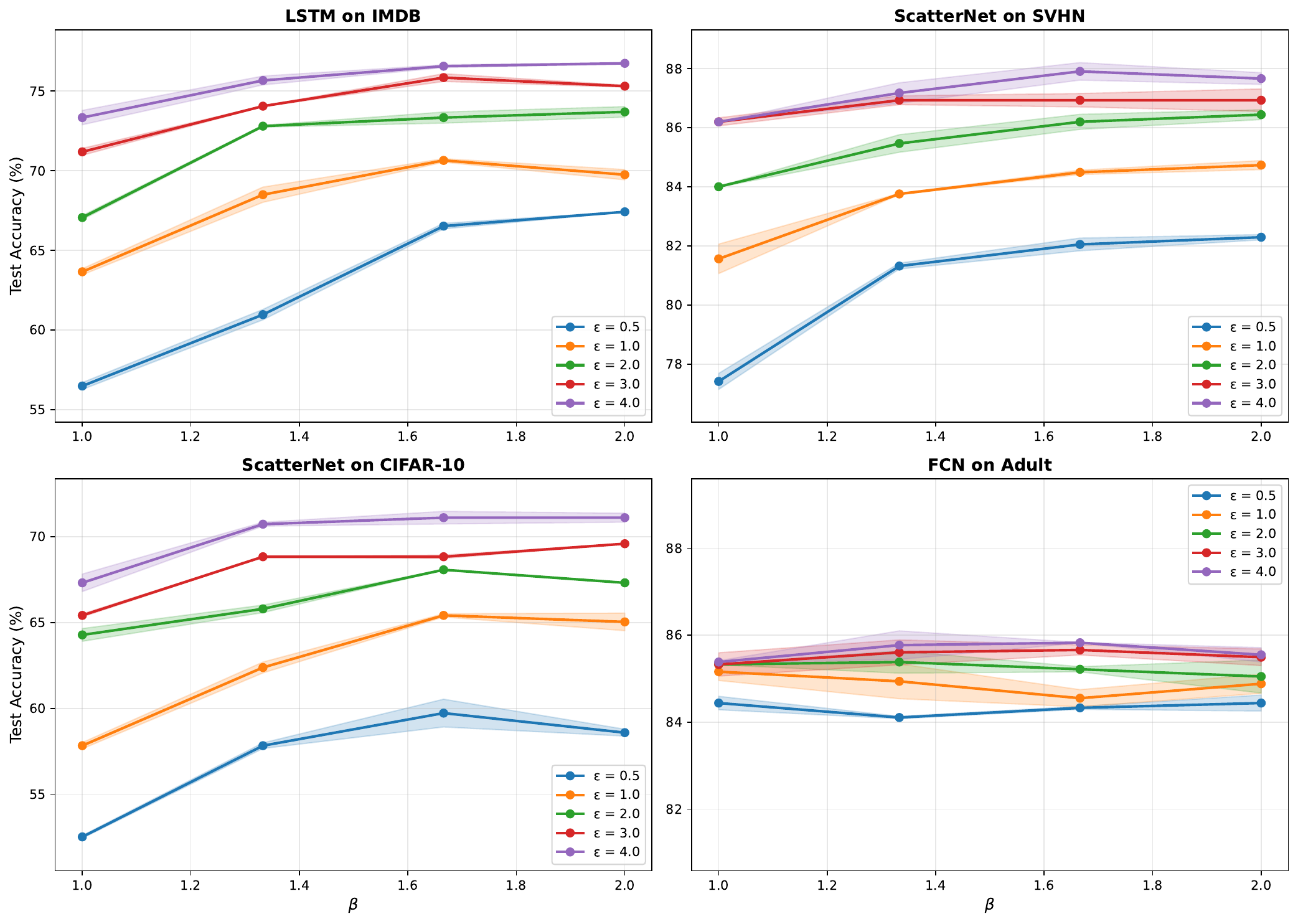}}
        \caption{\small{$\beta$-\DPSGD~results with clipping norm 0.25}}
        \label{fig:dpsgd_results_scatter_separate__clipping3}
    \end{figure}

\subsubsection{Noise Multiplier}\label{app.noisemult}

Here we explore the impact of varying the noise multiplier $\sigma$ on the test-accuracy, as a function of $\beta$. Figures \ref{fig:dpsgd_results_scatter_separate__noise_mult2}, \ref{fig:dpsgd_results_scatter_separate__noise_mult3}, and \ref{fig:dpsgd_results_scatter_separate__noise_mult4} show test-accuracy results of $\beta$-\DPSGD\ on all four databases with noise multipliers of 1.5, 2.0, 2.5, and 3.0. We observe that varying the noise multiplier has a relatively weak effect on accuracy, although for larger values of $\sigma$, the test accuracy curve (as a function of $\beta$) is flatter, with less dependence on $\beta$ -- particularly for ScatterNet and LSTM. 

As the noise multiplier increases, the privacy budget consumption per step decreases --- this means that for fixed $(\epsilon,\delta)$, the total number of training steps in training increases. For a fixed noise multiplier, increasing the value of $\beta$, increases the privacy cost for a single step --- this means that for a fixed $\sigma$, increasing $\beta$ decreases the total number of training steps. While existing work found that models attain higher accuracies when larger batch sizes are trained for more training epochs \cite{deepmind_jax_privacy}, we do not find any such evidence. Introducing the $\beta$ parameter appears to remove this effect.

Empirically, across all values of noise multipliers evaluated, $\beta=2$ performs as well as or better than other values in the range $\beta \in [1,2]$ on all datasets and $\epsilon$ values evaluated.

    \begin{figure}[H]
        \centering
         \subfigure{\includegraphics[width=.8\textwidth]{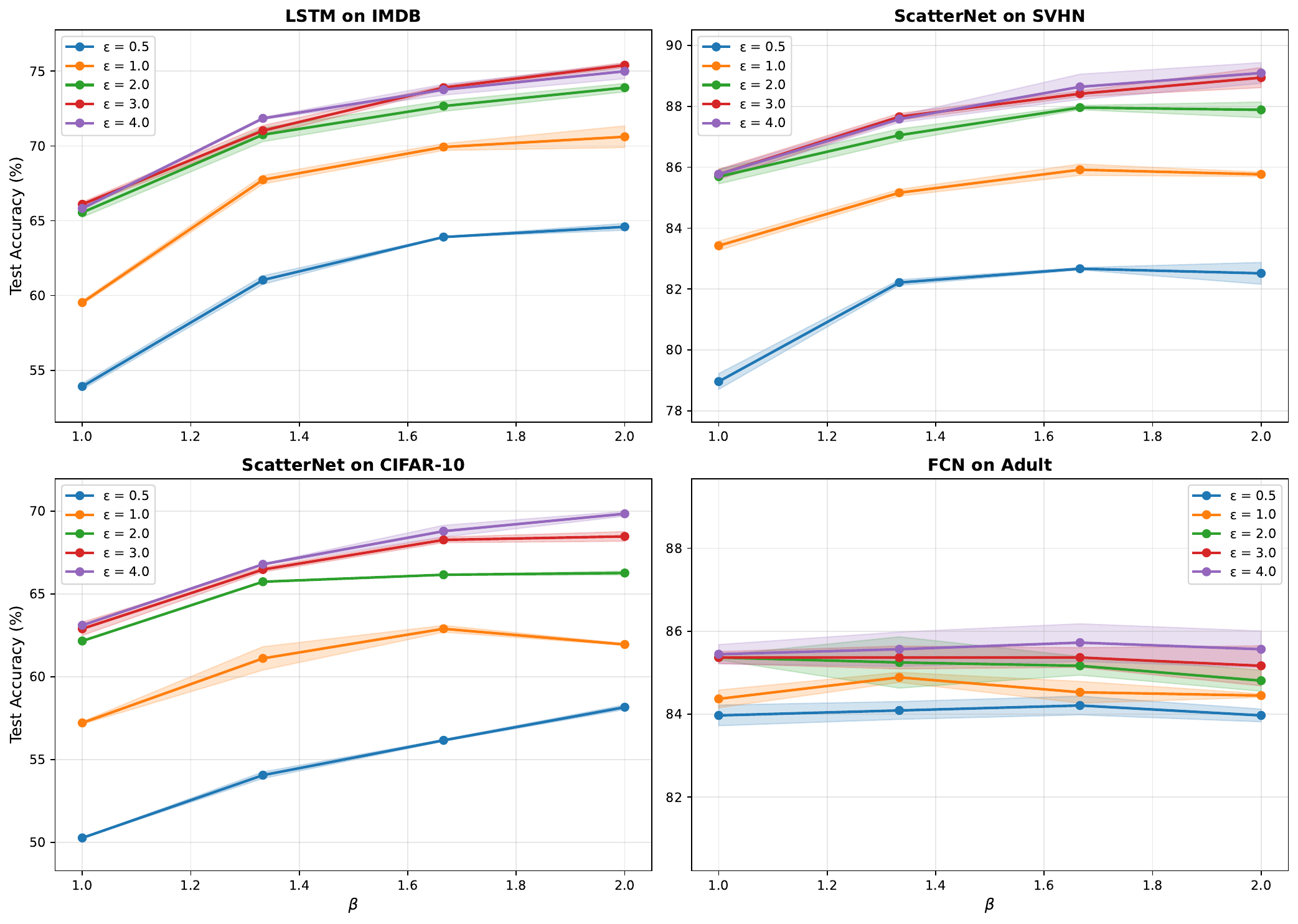}}
        \caption{\small{$\beta$-\DPSGD~results with noise multiplier $\sigma = 2.0$ }}
        \label{fig:dpsgd_results_scatter_separate__noise_mult2}
    \end{figure}

    \begin{figure}[H]
        \centering
         \subfigure{\includegraphics[width=.8\textwidth]{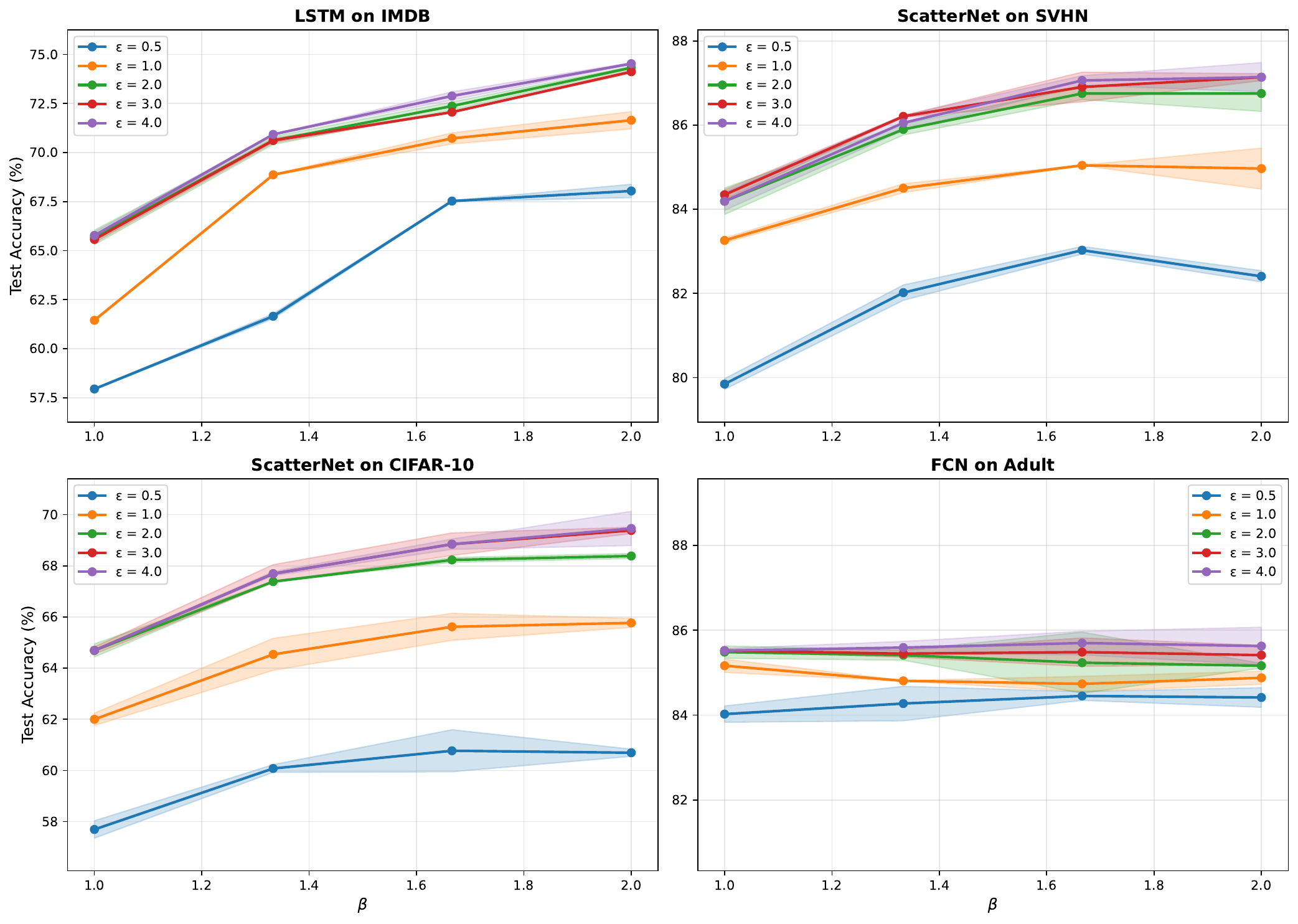}}
        \caption{\small{$\beta$-\DPSGD~results with noise multiplier $\sigma = 2.5$}}
        \label{fig:dpsgd_results_scatter_separate__noise_mult3}
    \end{figure}

    \begin{figure}[H]
        \centering
         \subfigure{\includegraphics[width=.8\textwidth]{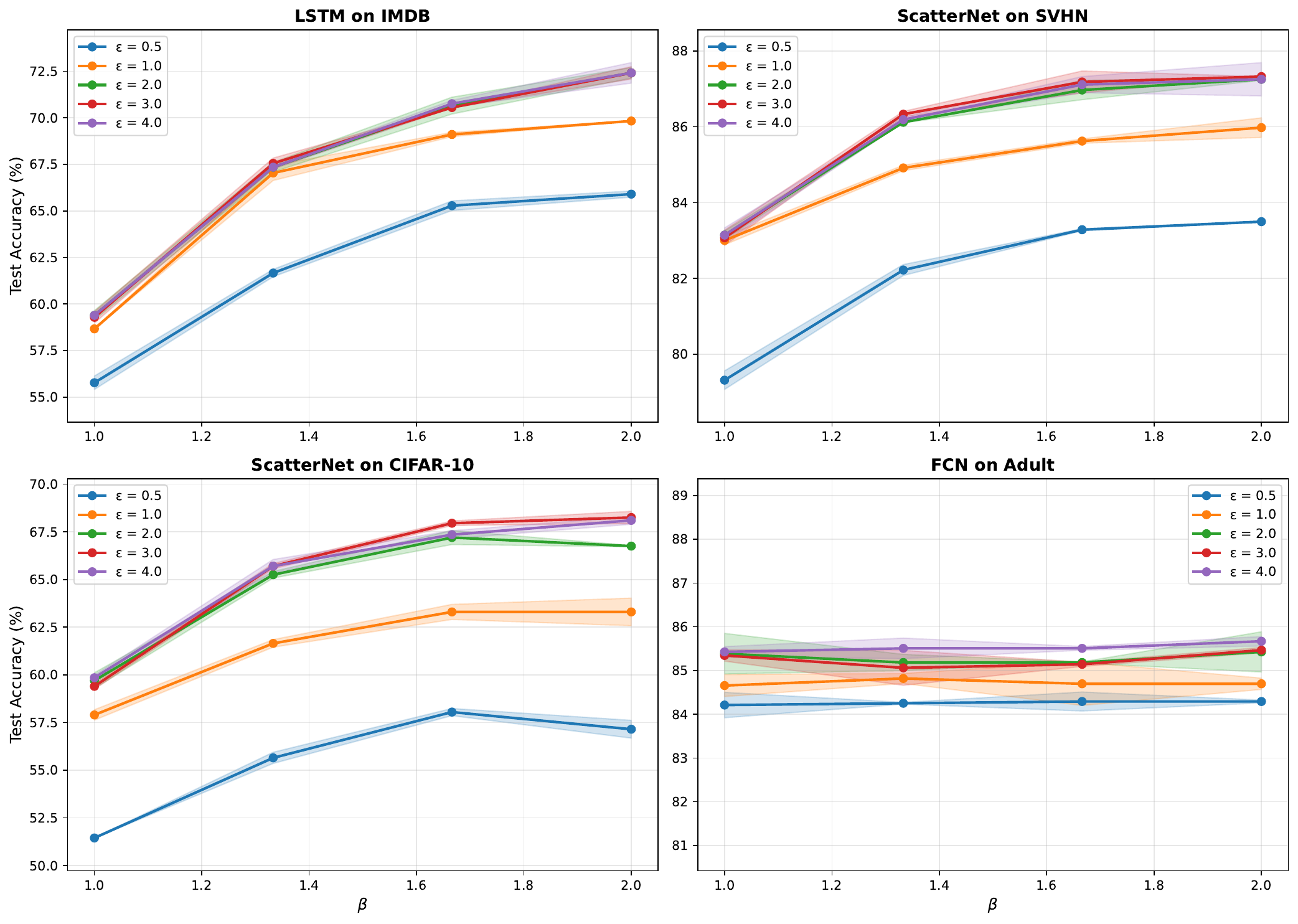}}
        \caption{\small{$\beta$-\DPSGD~results with noise multiplier $\sigma=3.0$ }}
        \label{fig:dpsgd_results_scatter_separate__noise_mult4}
    \end{figure}

\subsection{Sampling from the Generalized Gaussian Distribution} \label{appendix:ggd_fast_sampling}

Unlike the Gaussian and \Laplace~distributions, sampling from the Generalized Gaussian mechanism is not natively supported by built-in libraries like Python's `math' library, or the commonly used \href{https://numpy.org/}{numpy} library. Another commonly used library for statistical computing, \href{https://scipy.org/}{SciPy}, does have the `\href{https://docs.scipy.org/doc/scipy/reference/generated/scipy.stats.gennorm.html#scipy-stats-gennorm}{scipy.stats.gennorm}' function; however, we found that it regularly takes too long for intensive computations like stochastic gradient descent in practical settings, which involves sampling from high-dimensional gradients thousands of times. Further, the Scipy function is only able to be sampled on a CPU, which makes it ill-suited for DP-SGD, which is regularly performed on a GPU.

We implement a method for sampling from the Generalized Gaussian mechanism included in our code here: \url{https://github.com/RoyRin/Generalized-Gaussians}. In our experiments, we can sample from the Generalized Gaussian only $\sim 1.3$x slower than sampling from a Gaussian directly. It is possible to conduct similar sampling using the method of inverse probability transforms, since the Generalized Gaussian has a known CDF.

\subsection{Reproducibility and Computing Resources}

For our \DPSGD~experiments, the execution of our techniques does not result in a significant increase in processing time compared to the conventional application of DP-SGD. The only addition to the computation duration comes from increased amounts of hyperparameter searching. All experiments and data analysis are reproducible in the codebase provided 
\url{https://github.com/RoyRin/Generalized-Gaussians}. The \DPSGD\ results in this paper were completed in under 500 hours of GPU time, which was split across 16 machines that were mounted on Nvidia T4 and RTX6000 machines GPUs. All data analysis was conducted on a 8 core machine with 16 GB RAM machine.